\PassOptionsToPackage{numbers, compress}{natbib}
\documentclass{article}
\newcommand{\treportflag}{1} 

\ifnum\treportflag=1
  \usepackage[preprint]{neurips_2026}
\else
  \usepackage{neurips_2026}
\fi

\usepackage[utf8]{inputenc} 
\usepackage[T1]{fontenc}    
\usepackage{hyperref}       
\usepackage{url}            
\usepackage{booktabs}       
\usepackage{amsfonts}       
\usepackage{nicefrac}       
\usepackage{xcolor}         
\usepackage{colortbl}
\usepackage{drago}
\usepackage{setspace}
\usepackage[normalem]{ulem}
\usepackage[on]{editing}
\usepackage{tcolorbox}

\newcommand{\ttl}{Causal Bias Detection in Generative Artificial Intelligence}

\title{\ttl{}}

\author{%
  Drago Ple{\v c}ko \\
  Department of Statistics \& Data Science \\ 
  UCLA
}

\usepackage{selectp}

\begin{document}
\ifnum\treportflag=1
\fi

\maketitle

\begin{abstract}
  Automated systems built on artificial intelligence (AI) are increasingly deployed across high-stakes domains, raising critical concerns about fairness and the perpetuation of demographic disparities that exist in the world. 
In this context, causal inference provides a principled framework for reasoning about fairness, as it links observed disparities to underlying mechanisms and aligns naturally with human intuition and legal notions of discrimination. 
Prior work on causal fairness primarily focuses on the standard machine learning setting, where a decision-maker constructs a single predictive mechanism $f_{\widehat Y}$ for an outcome variable $Y$, while inheriting the causal mechanisms of all other covariates from the real world. 
The generative AI setting, however, is markedly more complex: generative models can sample from arbitrary conditionals over any set of variables, implicitly constructing their own beliefs about all causal mechanisms rather than learning a single predictive function. This fundamental difference requires new developments in causal fairness methodology.
We formalize the problem of causal fairness in generative AI and unify it with the standard ML setting under a common theoretical framework. We then derive new causal decomposition results that enable granular quantification of fairness impacts along both (a) different causal pathways and (b) the replacement of real-world mechanisms by the generative model's mechanisms. We establish identification conditions and introduce efficient estimators for causal quantities of interest, and demonstrate the value of our methodology by analyzing race and gender bias in large language models across different datasets.
\end{abstract}

\section{Introduction}
Automated systems built on machine learning and artificial intelligence are increasingly used to support or replace human decision-making across domains such as hiring, admissions, lending, law enforcement, and health care \citep{khandani2010consumer,mahoney2007method,brennan2009evaluating}. These systems now range from structured prediction models to modern generative AI, including large language and vision models. Across this spectrum, society is increasingly concerned about how automated systems compare to existing human decision processes, and whether they may perpetuate or amplify inequities between demographic groups. Prior work documents significant biases in decision-support tools for sentencing \citep{ProPublica}, face detection \citep{pmlr-v81-buolamwini18a}, online advertising \citep{Sweeney13,Datta15}, and authentication \citep{Sanburn15}, as well as demographic drift and occupational stereotypes in generative models \citep{luccioni2023stable, naik2023social, nangia2020crows, de2019bias, hendricks2018women, wang2024vlbiasbench}. Importantly, these issues are not unique to machines: the world from which data is sampled exhibits longstanding disparities, from gender pay gaps \citep{blau1992gender,blau2017gender} to racial bias in criminal sentencing \citep{sweeney1992influence,pager2003mark}. As a result, AI systems trained on observational data learn patterns shaped by historical processes, posing a fundamental question: under what conditions do AI systems replicate, exacerbate, or mitigate existing disparities?

In light of the fairness challenges described above, a growing literature has explored approaches to mitigating bias in automated systems. Here, we specifically mention the causal approach \citep{kusner2017counterfactual, kilbertus2017avoiding, nabi2018fair, zhang2018fairness, zhang2018equality, wu2019pc, chiappa2019path, plevcko2020fair, plecko2022causal}, which has two major benefits. 
First, causal notions allow for human-understandable and interpretable definitions and metrics of fairness, which are tied to the causal mechanisms transmitting the change between groups. Secondly, they offer a language that is aligned with the legal notions of discrimination, such as the disparate impact doctrine \citep{barocas2016big}.

Within the context of causal fairness, it is useful to distinguish between different tasks \citep{plecko2022causal}, namely (1) bias detection and quantification for existing outcomes or decision policies; (2) construction of fair predictions of an outcome; (3) construction of fair decision-making policies that are intended to be implemented in the real-world. A prototypical setting for fair prediction, known in prior work as the standard fairness model (SFM) \citep{plecko2022causal}, is shown in Fig.~\ref{fig:sfm}. Here, $X$ represents the protected attribute, $Z$ a possibly multidimensional set of confounders, $W$ a set of mediators, and $Y$ the outcome variable, while $\widehat{Y}$ represents an automated predictor. The causal approach aims to understand how different causal effects from $X$ to $Y$ (direct $X \to Y$, indirect $X \to W \to Y$, and spurious $X \bidir Z \to Y$) are inherited by $\widehat Y$ \cite{plecko2023reconciling}, and how to correct for such bias.

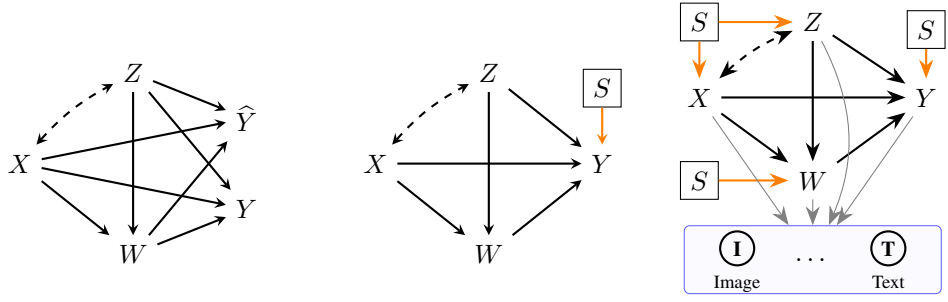
\begin{figure}[t]
    \centering
    \begin{subfigure}{0.32\textwidth}
    \centering
    \begin{tikzpicture}[>=stealth, rv/.style={thick}, rvc/.style={triangle, draw, thick, minimum size=7mm}, node distance=18mm]
	 \pgfsetarrows{latex-latex};

\node (e) at (0,3) {};

\node[rv] (v1) at (1,4.7) {$X$};
\node[rv] (v2) at (2.5,5.9) {$Z$};
\node[rv] (v3) at (2.5,3.5) {$W$};
\node[rv] (v4) at (4,4.1) {$Y$};
\node[rv] (v5) at (4,5.3) {$\widehat Y$};

\draw[<->,thick, dashed] (v1) edge[bend left=10] (v2);
\draw[->,thick] (v1) -- (v3);
\draw[->,thick] (v2) -- (v3);
\draw[->,thick] (v1) -- (v4);
\draw[->,thick] (v2) -- (v4);
\draw[->,thick] (v3) -- (v4);
\draw[->,thick] (v1) -- (v5);
\draw[->,thick] (v2) -- (v5);
\draw[->,thick] (v3) -- (v5);
\end{tikzpicture}
    \caption{Standard Fairness Model \citep{plecko2022causal}.}
    \label{fig:sfm}
    \end{subfigure}
    \hfill
    \begin{subfigure}{0.32\textwidth}
    \centering
    \begin{tikzpicture}[>=stealth, rv/.style={thick}, rvc/.style={draw,minimum size=5mm,inner sep=1pt}, node distance=18mm]
\pgfsetarrows{latex-latex};

\node (e) at (0,3) {};

\node[rv] (v1) at (1,4.7) {$X$};
\node[rv] (v2) at (2.5,5.9) {$Z$};
\node[rv] (v3) at (2.5,3.5) {$W$};
\node[rv] (v4) at (4,4.7) {$Y$};

\node[rvc] (v5) at (4,5.7) {$S$};

\draw[<->,thick, dashed] (v1) edge[bend left=10] (v2);
\draw[->,thick] (v1) -- (v3);
\draw[->,thick] (v1) -- (v4);
\draw[->,thick] (v2) -- (v3);
\draw[->,thick] (v2) -- (v4);
\draw[->,thick] (v3) -- (v4);
\draw[->,thick, orange] (v5) -- (v4);

\end{tikzpicture}
    \caption{S-SFM for machine learning.}
    \label{fig:sfm-ml}
    \end{subfigure}
    \hfill
    \begin{subfigure}{0.32\textwidth}
    \centering
    \begin{tikzpicture}[x=1cm,y=1cm,>={Stealth[length=2.5mm]}]
\tikzset{
  var/.style={circle,fill=black,inner sep=1.4pt},
  mod/.style={circle,draw,thick,inner sep=0pt,minimum size=4.5mm,font=\bfseries\footnotesize},
  lbl/.style={font=\scriptsize},
  rv/.style={thick},
  rvc/.style={draw,minimum size=5mm,inner sep=1pt}
}

\fill[blue!2,rounded corners=2pt] (0.8,3) rectangle (4.2,2.1);
\draw[blue!60,rounded corners=2pt] (0.8,3) rectangle (4.2,2.1);

\node[inner sep=0pt, fit={(0.8,3) (4.2,2.1)}] (rectnode) {};

\node[rv] (v1) at (1,4.7) {$X$};
\node[rv] (v2) at (2.5,5.7) {$Z$};
\node[rv] (v3) at (2.5,3.6) {$W$};
\node[rv] (v4) at (4,4.7) {$Y$};
\node[rvc] (sxz) at (1,5.7) {$S$};
\node[rvc] (sw) at (1,3.6) {$S$};

\node[rvc] (sy) at (4,5.6) {$S$};

\draw[<->,thick, dashed] (v1) edge[bend left=10] (v2);
\draw[->,thick] (v1) -- (v3);
\draw[->,thick] (v2) -- (v3);
\draw[->,thick] (v1) -- (v4);
\draw[->,thick] (v2) -- (v4);
\draw[->,thick] (v3) -- (v4);

\draw[->,thick, orange] (sxz) -- (v1);
\draw[->,thick, orange] (sxz) -- (v2);

\draw[->,thick, orange] (sw) -- (v3);
\draw[->,thick, orange] (sy) -- (v4);

\node[mod] (I) at (1.5,2.7) {I};
\node[lbl,below=0mm of I] {Image};

\node (dots) at (2.5, 2.55) {$\dots$};

\node[mod] (T) at (3.5,2.7) {T};
\node[lbl,below=0mm of T] {Text};

\draw[->, gray] (v1) -- (rectnode);
\draw[->, gray] (v2) edge[bend left=25] (rectnode);
\draw[->, gray] (v3) -- (rectnode);
\draw[->, gray] (v4) -- (rectnode);
\end{tikzpicture}
    \caption{S-SFM for generative AI.}
    \label{fig:sfm-genai}
\end{subfigure}
\label{fig:new-sfms}
\caption{Standard Fairness Model (SFM) for machine learning and generative AI settings.}
\vspace{-0.1in}
\end{figure}

The setting of generative AI, however, is markedly more complex, along several dimensions. 
First, we most commonly do not have access to data on covariates $(X, Z, W, Y)$ directly, but rather have access to lower-level data of different modalities, such as images or text, whose dynamics and structure are influenced by a causal model (see Fig.~\ref{fig:sfm-genai}).
Second, when working with tabular data, as in the standard ML setting, the decision-maker is constructing a single mechanism $f_{\widehat{Y}}$ that maps covariates $(X, Z, W)$ to a prediction $\widehat{Y}$, while all other mechanisms $f_X, f_Z, f_W$ are \textit{inherited from the real-world}. 
In generative AI, this is no longer the case, and we are (implicitly) constructing a generative model that can map an arbitrary set of variables $\{ V_1, \dots, V_k \}$ to another arbitrary set of variables $\{V_{k+1}, \dots, V_\ell$\}. As a consequence, a generative model $M$ may have its own beliefs about the relationship between variables $X, Z$, and $W$ -- that is, it constructs new mechanisms $f^M_X, f^M_Z, f^M_W$ -- entirely opposite to the case with tabular data where such mechanism were determined from the real world. 
The final challenge stems from the fact that generative models do not have access to the correct variable correlations in the real world, as recent work demonstrates \citep{plecko2025epidemiology}. Thus, variable relationships in the generative model could be quite different from the real-world relationships. 

These challenges illustrate that a fresh approach is needed for causal fairness in generative AI, which is the topic of this paper. In this context, our contributions are as follows:
\begin{enumerate}[label=(\roman*)]
    \item We formalize the problem of causal fairness in generative AI, showing how it differs from the standard ML setting -- and bring the two settings under the same theoretical umbrella,
    \item We derive new causal decomposition results (Thm.~\ref{thm:new-decomp}) -- allowing for a more granular quantification of fairness impacts along (a) different causal pathways and (b) arising from different mechanisms (generative model vs. real world),
    \item We establish identification conditions, and introduce efficient estimators for causal quantities of interest (Prop.~\ref{prop:id-est}), 
    \item We apply our methodology to analyze LLMs for (a) racial bias in substance abuse; (b) racial bias in chronic disease burden; (c) sex bias in income levels.
\end{enumerate}
We now provide some intuition for the developments in this manuscript. In our framework, the standard causal fairness ML setting, shown in Fig.~\ref{fig:sfm}, can be modeled as in Fig.~\ref{fig:sfm-ml}. Here, the node labeled $S$ indicates whether the outcome variable $Y$ is sampled according to the real-world mechanism $f_Y^{rw}$, or according to the ML model $f_Y^{m}$. In this way, the $S$-node can be used to denote from which generative environment (real world or model) the data is sampled. In context of generative AI, as mentioned earlier, the $S$-node must point to all covariates $X, Z, W$, and $Y$, since generative models are able to construct their own beliefs about the relationship between variables.

\paragraph{Relationship to Previous Literature.}
Our work is related to the literature on causal fairness \citep{plecko2022causal}, much of which is focused on the typical ML setting \citep{kusner2017counterfactual, chiappa2019path}. Our approach builds on some of these prior works, and extends the scope to generative AI settings. 
Secondly, our work is related to approaches for achieving counterfactual fairness in text/image classification and generation \citep{cheong2022counterfactual, jung2024counterfactually, garg2019counterfactual, kim2021counterfactual, joo2020gender}. However, such works usually focus only on counterfactual fairness, and fail to model the latent causal process and quantify granularly which causal mechanisms contribute to disparities. Finally, we mention existing work that performs statistical quantification of disparities in generative AI \citep{luccioni2023stable, naik2023social, de2019bias, hendricks2018women}, including dedicated bias benchmarks \citep{nadeem2021stereoset, nangia2020crows, liang2022holistic, smith2022m, wang2024vlbiasbench}. Our approach can be used to analyze the same phenomena through a different, causal lens, offering insight into \textit{which mechanisms} drive the disparities.

\subsection{Preliminaries}
We use the language of structural causal models (SCMs) as our basic semantical framework \citep{pearl:2k}. A structural causal model (SCM) is
a tuple $\mathcal{M} := \langle V, U, \mathcal{F}, P(u)\rangle$ , where $V$, $U$ are sets of
endogenous (observable) and exogenous (latent) variables 
respectively, $\mathcal{F}$ is a set of functions $f_{V_i}$,
one for each $V_i \in V$, where $V_i \gets f_{V_i}(\pa(V_i), U_{V_i})$ for some $\pa(V_i)\subseteq V$ and
$U_{V_i} \subseteq U$. $P(u)$ is a strictly positive probability measure over $U$. Each SCM $\mathcal{M}$ is associated to a causal diagram $\mathcal{G}$ over the node set $V$, where $V_i \rightarrow V_j$ if $V_i$ is an argument of $f_{V_j}$, and $V_i \bidir V_j$ if the corresponding $U_{V_i}, U_{V_j}$ are not independent. An instantiation of the exogenous variables $U = u$ is called a \textit{unit}. By $Y_{x}(u)$ we denote the potential response of $Y$ when setting $X=x$ for the unit $u$, which is the solution for $Y(u)$ to the set of equations obtained by evaluating the unit $u$ in the submodel $\mathcal{M}_x$, in which all equations in $\mathcal{F}$ associated with $X$ are replaced by $X = x$. Furthermore, we use existing notions of direct, indirect, and spurious effects \citep{zhang2018fairness, plecko2022causal}: $x\text{-DE}_{x_0, x_1}(y\mid x) = P(y_{x_1, W_{x_0}} \mid x) - P(y_{x_0}\mid x)$, $x\text{-IE}_{x_1, x_0}(y\mid x) = P(y_{x_1, W_{x_0}}\mid x) - P(y_{x_1} \mid x)$, and $x\text{-SE}_{x_1, x_0}(y) = P(y_{x_1} \mid x_0) - P(y_{x_1} \mid x_1)$. Based on these, the TV$_{x_0, x_1}(y)$ measure $\ex[Y \mid X = x_1] - \ex[Y \mid X=x_0]$ can be decomposed as:
    \begin{align}
        \text{TV}_{x_0, x_1}(y) &=  x\text{-DE}_{x_0, x_1}(y\mid x_0) 
        - {x\text{-IE}_{x_1, x_0}(y\mid x_0)}
        -x\text{-SE}_{x_1, x_0}(y).
    \end{align}
A graphical representation for the $x$-specific direct effect is shown in Fig.~\ref{fig:x-de}. The left side represents the potential outcome $Y_{x_1, W_{x_0}} \mid X = x_0$, where $X=x_1$ along the direct path $X \to Y$, while $X = x_0$ along the indirect path $X\to W \to Y$, and $X=x_0$ along back-door paths between $X$ and $Y$. The right side represents the potential outcome $Y_{x_0, W_{x_0}} \mid X = x_0$, where $X = x_0$ along all three causal paths (direct, indirect, spurious). Therefore, $x$-DE measures the effect of a $x_0 \to x_1$ transition along the direct effect, for individuals with $X = x_0$, while $X = x_0$ along the indirect path. 
\begin{figure}
\centering

\begin{subfigure}{0.51\linewidth}
\centering
\begin{tikzpicture}
 [>=stealth, rv/.style={thick}, rvc/.style={triangle, draw, thick, minimum size=7mm}, node distance=18mm]
 \pgfsetarrows{latex-latex};

 \node[rv] (z1) at (0,1) {$Z$};
 \node[rv] (x1) at (-1.5,0) {$x_0$};
 \node[rv] (x1x) at (-0.9,0) {$x_1$};
 \node[rv] (w1) at (0,-1) {$W$};
 \node[rv] (y1) at (1.5,0) {$Y$};

 \draw[->] (x1) -- (w1);
 \draw[->] (z1) -- (y1);
 \path[->] (x1x) edge[bend left=0] (y1);
 \path[<->] (x1) edge[bend left=30, dashed] (z1);
 \draw[->] (w1) -- (y1);
 \draw[->] (z1) -- (w1);

 \node (mns) at (2.25, 0) {\Large $-$};
 \node (a) at (0, -1.5) {$P(y_{x_1, W_{x_0}} \mid x_0)$};
 \node (b) at (4.5, -1.5) {$P(y_{x_0} \mid x_0)$};

 \node[rv] (z2) at (4.5,1) {$Z$};
 \node[rv] (x2) at (3,0) {$x_0$};
 \node[rv] (w2) at (4.5,-1) {$W$};
 \node[rv] (y2) at (6,0) {$Y$};

 \draw[->] (x2) -- (w2);
 \draw[->] (z2) -- (y2);
 \path[->] (x2) edge[bend left=0] (y2);
 \path[<->] (x2) edge[bend left=30, dashed] (z2);
 \draw[->] (w2) -- (y2);
 \draw[->] (z2) -- (w2);

\end{tikzpicture}
\caption{$x$-specific direct effect.}
\label{fig:x-de}
\end{subfigure}
\hfill
\begin{subfigure}{0.48\linewidth}
\centering
\begin{tikzpicture}
 [>=stealth, rv/.style={thick}, rvc/.style={triangle, draw, thick, minimum size=4mm}, node distance=10mm]
 \pgfsetarrows{latex-latex};

 \node[rv] (z) at (0,1) {$Z$};

 \node[rv] (xw) at (-1.5,-0.45) {$x_w$};
 \node[rv] (xz) at (-1.5,0.5) {$X=x_z$};
 \node[rv] (xy) at (-0.8,-0.1) {$x_y$};

 \node[rv] (sz) at (-1.5, 1.75) {\fbox{$s_z$}};
 \node[rv] (sw) at (-1.5, -1) {\fbox{$s_w$}};
 \node[rv] (sy) at (1.5, 1) {\fbox{$s_y$}};

 \node[rv] (w) at (0,-1) {$W$};
 \node[rv] (y) at (1.5,0) {$Y$};

 \draw[->] (xw) -- (w);
 \draw[->] (xy) -- (y);
 \path[<->] (xz) edge[bend left=30, dashed] (z);

 \draw[->] (z) -- (y);
 \draw[->] (z) -- (w);
 \draw[->] (w) -- (y);

 \draw[->,thick, orange] (sz) -- (xz);
 \draw[->,thick, orange] (sz) -- (z);
 \draw[->,thick, orange] (sw) -- (w);
 \draw[->,thick, orange] (sy) -- (y);

\end{tikzpicture}
\caption{A generic potential outcome in $S$-SFM.}
\label{fig:ssfm-po}
\end{subfigure}

\caption{Graphical models for (a) $x$-specific direct effect; (b) generic potential outcome in $S$-SFM.}
\label{fig:combined}
\end{figure}
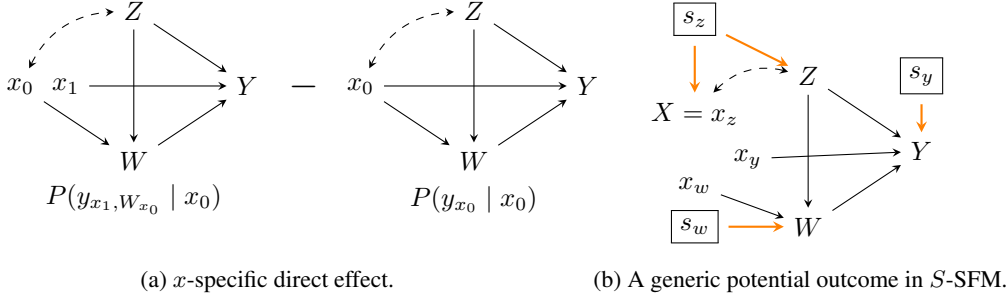
Similar interpretations can be given for the indirect and spurious effects. The $x$-specific indirect effect captures the reverse transition $x_1 \to x_0$ along the path $X \to W \to Y$; and the spurious effect quantifies how conditioning on $X=x_0$ compares to conditioning on $X=x_1$, for the potential outcome $Y_{x_1}$ where $X$ is set to $x_1$ by intervention. 
\section{Causal Fairness in Generative AI} \label{sec:cfgai}
We start by introducing the required structural semantics for the setting of fairness in generative AI. In particular, as noted in the introduction, our modeling approach takes into account that data generation happens in two distinct environments, namely the real-world (denoted by $S=rw$, or $S=0$) and the generative model ($S=gm$, or $S=1$). For each of these two environments, we can write down explicitly the underlying structural causal models $M^{rw}, M^{gm}$ over variables $X, Z, W, Y$:

\begin{minipage}{.45\linewidth}
\begin{empheq}[left={M^{rw} = \empheqlbrace}]{align} 
  X, Z &\gets f^{rw}_{X,Z}(U_{XZ})  \label{eq:f-xz-rw}\\
  W &\gets f^{rw}_W(X, Z, U_W) \\
  Y &\gets f^{rw}_Y(X, Z, W, U_Y)\label{eq:f-y-rw} \\
  U &\sim P^{rw}(U)
\end{empheq}
\end{minipage}%
\hfill \text{, } \hfill
\begin{minipage}{.48\linewidth}
\vspace{-0.14in}
\begin{empheq}[left={M^{gm} = \empheqlbrace}]{align}
  X, Z &\gets f^{gm}_{X,Z}(U_{XZ})  \label{eq:f-xz-gm}\\
  W &\gets f^{gm}_W(X, Z, U_W) \\
  Y &\gets f^{gm}_Y(X, Z, W, U_Y)\label{eq:f-y-gm} \\
  U &\sim P^{gm}(U).
\end{empheq} 
\end{minipage}

\vspace{0.1in}
\noindent Here, variables $X, Z$ are put together in a single mechanism $f_{X,Z}$ that depends on the noise variable $U_{XZ}$, whereas $f_W, f_Y$ mechanisms depend on their causal parents and respective noise variables. Importantly, we note that the mechanisms in the real world and the generative model may be distinct. To put the two models $M^{rw}, M^{gm}$ under the same umbrella, we introduce a type of selection node \citep{pearl2011transportability}, labeled $S$, which influences all of the covariates (see Fig.~\ref{fig:sfm-genai}). 
In particular, the node $S$ encodes whether data is generated from the real world mechanism $f^{rw}_{V_i}$ (and it's noise distribution $P^{rw}(U_{V_i})$) or the generative model's mechanism $f^{gm}_{V_i}$ (and it's noise distribution). Crucially, this allows us to represent the two SCMs jointly, 
\begin{empheq}[left={M = \empheqlbrace}]{align}
  S &\sim P(S=rw)=p, P(S=gm)=1-p \\
  X, Z &\gets f^{S}_{X,Z}(U_{XZ})  \label{eq:f-xz}\\
  W &\gets f^{S}_W(X, Z, U_W) \\
  Y &\gets f^{S}_Y(X, Z, W, U_Y)\label{eq:f-y} \\
  U &\sim P^{S}(U).
\end{empheq}
In other words, the variable $S$ determines which mechanisms (and which noise variables) are used for determining the value of $(X, Z)$, $W$, and $Y$ variables. When $S=0$, we have the generating process describing the real world, while $S=1$ leads to the generating process of the AI model. 
Formally, the $S$-node consists of separate selection nodes $S = (S_{XZ}, S_W, S_Y)$ for different variables, and this allows us to consider mechanism replacements independently \citep[Ch.~10]{bareinboim2025causalai}. Nonetheless, we write $S$ without subscript in most places, as the subscript can often be inferred from context.
The SCM described above also has a corresponding graphical representation which we call $S$-SFM:
\begin{definition}[$S$-SFM]
    The $S$-standard fairness model is the SFM model with an $S$-node pointing to each of the covariate sets $\{X, Z\}, W$ and $Y$, as shown in Fig.~\ref{fig:sfm-genai}. $\hfill\square$
\end{definition}
The usefulness of the above semantics comes from our ability to consider different interventions on the $S$-node along the edges $S\to \{X,Z\}$, $S \to W$, and $S \to Y$. We now introduce the notation for such potential outcomes. We write
\begin{align}
    Y^{s_y}_{x_y, W^{s_w}_{x_w}} \mid X = x_z, S = s_z,
\end{align}
for the potential outcome visualized in Fig.~\ref{fig:ssfm-po}. Specifically, the value $S=s_z$ determines whether the mechanism $f_{X,Z}$ corresponds to the real-world mechanism or the generative model's mechanism; while the $X=x_z$ determines the value of $X$ that propagates along spurious paths from $X$ to $Y$. Further, $S=s_w$ determines whether $W$ is drawn from the real-world mechanism, or the generative model one; while $X=x_w$ determines the value of $X$ along the indirect path $X\to W \to Y$. Finally, $S=s_y$ controls whether $Y$ is drawn from the real-world mechanism, or the generative model one, while $X=x_y$ determines the value of $X$ along the direct edge $X\to Y$. 
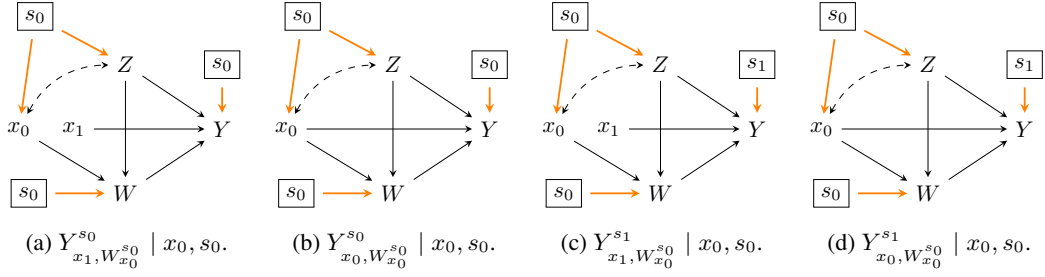
\begin{figure}
    \centering
    \begin{subfigure}{0.24\textwidth}
        \centering
        \resizebox{\linewidth}{!}{\begin{tikzpicture}
        [>=stealth, rv/.style={thick}, rvc/.style={triangle, draw, thick, minimum size=4mm}, node distance=10mm]
        \pgfsetarrows{latex-latex};
        \node[rv] (z) at (0,1) {$Z$};
        
        \node[rv] (xz) at (-1.65,0.0) {$x_0$};
        
        \node[rv] (xy) at (-0.8,-0.0) {$x_1$};
        \node[rv] (sz) at (-1.4, 1.75) {\fbox{$s_0$}};
        \node[rv] (sw) at (-1.5, -1) {\fbox{$s_0$}};
        \node[rv] (sy) at (1.5, 1) {\fbox{$s_0$}};
        \node[rv] (w) at (0,-1) {$W$};
        \node[rv] (y) at (1.5,0) {$Y$};
        \draw[->] (xz) -- (w);
        \draw[->] (xy) -- (y);
        \path[<->] (xz) edge[bend left=30, dashed] (z);
        \draw[->] (z) -- (y);
        \draw[->] (z) -- (w);
        \draw[->] (w) -- (y);
        
        \draw[->,thick,orange] (sz) -- (xz);
        \draw[->,thick,orange] (sz) -- (z);
        \draw[->,thick,orange] (sw) edge[bend left=0] (w);
        \draw[->,thick,orange] (sy) edge[bend left=0] (y);
        \end{tikzpicture}}
        \caption{$Y^{s_0}_{x_1, W^{s_0}_{x_0}} \mid x_0, s_0$.}
    \end{subfigure}
    \hfill
    \begin{subfigure}{0.24\textwidth}
        \centering
        \resizebox{\linewidth}{!}{\begin{tikzpicture}
        [>=stealth, rv/.style={thick}, rvc/.style={triangle, draw, thick, minimum size=4mm}, node distance=10mm]
        \pgfsetarrows{latex-latex};
        \node[rv] (z) at (0,1) {$Z$};
        
        \node[rv] (xz) at (-1.65,0.0) {$x_0$};
        
        \node[rv] (sz) at (-1.4, 1.75) {\fbox{$s_0$}};
        \node[rv] (sw) at (-1.5, -1) {\fbox{$s_0$}};
        \node[rv] (sy) at (1.5, 1) {\fbox{$s_0$}};
        \node[rv] (w) at (0,-1) {$W$};
        \node[rv] (y) at (1.5,0) {$Y$};
        \draw[->] (xz) -- (w);
        \draw[->] (xz) -- (y);
        \path[<->] (xz) edge[bend left=30, dashed] (z);
        \draw[->] (z) -- (y);
        \draw[->] (z) -- (w);
        \draw[->] (w) -- (y);
        
        \draw[->,thick,orange] (sz) -- (xz);
        \draw[->,thick,orange] (sz) -- (z);
        \draw[->,thick,orange] (sw) edge[bend left=0] (w);
        \draw[->,thick,orange] (sy) edge[bend left=0] (y);
        \end{tikzpicture}}
        \caption{$Y^{s_0}_{x_0, W^{s_0}_{x_0}} \mid x_0, s_0$.}
    \end{subfigure}
    \hfill
    \begin{subfigure}{0.24\textwidth}
        \centering
        \resizebox{\linewidth}{!}{\begin{tikzpicture}
        [>=stealth, rv/.style={thick}, rvc/.style={triangle, draw, thick, minimum size=4mm}, node distance=10mm]
        \pgfsetarrows{latex-latex};
        \node[rv] (z) at (0,1) {$Z$};
        
        \node[rv] (xz) at (-1.65,0.0) {$x_0$};
        
        \node[rv] (xy) at (-0.8,-0.0) {$x_1$};
        \node[rv] (sz) at (-1.4, 1.75) {\fbox{$s_0$}};
        \node[rv] (sw) at (-1.5, -1) {\fbox{$s_0$}};
        \node[rv] (sy) at (1.5, 1) {\fbox{$s_1$}};
        \node[rv] (w) at (0,-1) {$W$};
        \node[rv] (y) at (1.5,0) {$Y$};
        \draw[->] (xz) -- (w);
        \draw[->] (xy) -- (y);
        \path[<->] (xz) edge[bend left=30, dashed] (z);
        \draw[->] (z) -- (y);
        \draw[->] (z) -- (w);
        \draw[->] (w) -- (y);
        
        \draw[->,thick,orange] (sz) -- (xz);
        \draw[->,thick,orange] (sz) -- (z);
        \draw[->,thick,orange] (sw) edge[bend left =0] (w);
        \draw[->,thick,orange] (sy) edge[bend left=0] (y);
        \end{tikzpicture}}
        \caption{$Y^{s_1}_{x_1, W^{s_0}_{x_0}} \mid x_0, s_0$.}
    \end{subfigure}
    \hfill
    \begin{subfigure}{0.24\textwidth}
        \centering
        \resizebox{\linewidth}{!}{\begin{tikzpicture}
        [>=stealth, rv/.style={thick}, rvc/.style={triangle, draw, thick, minimum size=4mm}, node distance=10mm]
        \pgfsetarrows{latex-latex};
        \node[rv] (z) at (0,1) {$Z$};
        
        \node[rv] (xz) at (-1.65,0.0) {$x_0$};
        
        \node[rv] (sz) at (-1.4, 1.75) {\fbox{$s_0$}};
        \node[rv] (sw) at (-1.5, -1) {\fbox{$s_0$}};
        \node[rv] (sy) at (1.5, 1) {\fbox{$s_1$}};
        \node[rv] (w) at (0,-1) {$W$};
        \node[rv] (y) at (1.5,0) {$Y$};
        \draw[->] (xz) -- (w);
        \draw[->] (xz) -- (y);
        \path[<->] (xz) edge[bend left=30, dashed] (z);
        \draw[->] (z) -- (y);
        \draw[->] (z) -- (w);
        \draw[->] (w) -- (y);
        
        \draw[->,thick,orange] (sz) -- (xz);
        \draw[->,thick,orange] (sz) -- (z);
        \draw[->,thick,orange] (sw) edge[bend left=0] (w);
        \draw[->,thick,orange] (sy) edge[bend left=0] (y);
        \end{tikzpicture}}
        \caption{$Y^{s_1}_{x_0, W^{s_0}_{x_0}} \mid x_0, s_0$.}
    \end{subfigure}
    \caption{Quantifying differences using $S$-SFM potential outcomes.}
    \label{fig:x-de-with-S}
\end{figure}
We now illustrate why such potential outcomes provide us with a useful tool for understanding fairness in generative AI systems. Consider the difference between the potential outcomes in (a) and (b) of Fig.~\ref{fig:x-de-with-S}. For both potential outcomes, $S=s_0$ for each mechanism, meaning that all the mechanisms are from the real world. The difference between the two potential outcomes lies in the value of $X$ along the direct path $X\to Y$, and thus captures the direct effect of a $x_0 \to x_1$ transition in the real world. 
We contrast this with the difference between (c) and (d) of Fig.~\ref{fig:x-de-with-S}, where the same direct effect of an $x_0 \to x_1$ transitions is captured, but in the world where the $f_Y$ mechanism is taken from the generative model ($s_1 \to Y$), while the $f_{X, Z}, f_W$ mechanisms are taken from the real world. Interestingly, note that the difference between the terms (referring to the labels in Fig.~\ref{fig:x-de-with-S})  
\begin{align}
    [(c) - (d)] - [(a) - (b)]
\end{align}
captures how much the direct effect changes when the real world mechanism $f_Y^{rw}$ is replaced by the generative model's mechanism $f_Y^{gm}$. In our analysis of fairness impacts of generative models, such quantities will play a key role, which motivates the following definition:
\begin{definition}[$S$-modified $x$-specific Effects]
The $S$-modified $x$-specific \{direct, indirect, spurious\} effects of $X$ on $Y$ are defined as:
    \begin{align}
        x\text{-DE}^{s_z, s_w, s_y}_{x_0, x_1}(y\mid x) &= \ex(Y^{s_y}_{x_1, W^{s_w}_{x_0}} \mid X = x, S=s_z) - \ex(Y^{s_y}_{x_0, W^{s_w}_{x_0}}\mid X = x, S = s_z)  \\
        x\text{-IE}^{s_z, s_w, s_y}_{x_1, x_0}(y\mid x) &= \ex(Y^{s_y}_{x_1, W^{s_w}_{x_0}} \mid X = x, S = s_z) - \ex(Y^{s_y}_{x_1, W^{s_w}_{x_1}} \mid X= x, S = s_z)\\
        x\text{-SE}^{s_z, s_w, s_y}_{x_1, x_0}(y) &= \ex(Y^{s_y}_{x_1, W^{s_w}_{x_1}} \mid X = x_0, S = s_z) - \ex(Y^{s_y}_{x_1, W^{s_w}_{x_1}} \mid X = x_1, S = s_z).
    \end{align}
    Based on these measures, we can further define second-order $S$-based differences,
    \begin{align}
        \Delta \text{CE}_{x, x'}(y) ^{s_z\to s_z', s_w\to s_w', s_y
    \to s_y'} = \text{CE}^{s'_z, s'_w, s'_y}_{x, x'}(y) -  \text{CE}^{s_z, s_w, s_y}_{x, x'}(y),
    \end{align}
    where CE stands for any of the $x$-specific effects \{DE, IE, SE\}. In case any $s_{V_i}=s_{V_i}'$, we simply indicate the $s_{V_i}$ value instead of a transition $s_{V_i} \to s_{V_i}'$. $\hfill\square$
\end{definition}
After introducing the key building blocks, we now describe how these elements can be used to analyze fairness in generative AI systems.

\subsection{Inferential Challenge}
A common starting point in investigating fairness is the marginal disparity in the outcome between groups, which in the real world can be measured as:
\begin{align}
    \text{TV}^{rw}_{x_0, x_1} = \ex^{rw}[Y \mid X = x_1] - \ex^{rw}[Y \mid X = x_0].
\end{align}
Similarly, we can quantify the same disparity in the generative model via:
\begin{align}
    \text{TV}^{gm}_{x_0, x_1} = \ex^{gm}[Y \mid X = x_1] - \ex^{gm}[Y \mid X = x_0].
\end{align}
The key goal of our approach is to provide a fine-grained understanding of how the disparity in the world $\text{TV}^{rw}_{x_0, x_1}$ changes to $\text{TV}^{gm}_{x_0, x_1}$ in the generative model. That is, our quantity of interest is $\Delta \text{TV}_{x_0, x_1}^{s_0 \to s_1}(y) \overset{\Delta}{=} \text{TV}^{gm}_{x_0, x_1} - \text{TV}^{rw}_{x_0, x_1}$, which can be analyzed based on the following result (all proofs are provided in App.~\ref{appendix:proofs}): 
\begin{theorem}[$S$-based TV Decomposition]
\label{thm:new-decomp}
    The difference between the real-world and generative model TV measures, $\Delta \text{TV}_{x_0, x_1}^{s_0 \to s_1}(y)$ can be decomposed as:
    \begin{align} \label{eq:delta-tv-decomp}
        \Delta \text{TV}_{x_0, x_1}^{s_0 \to s_1}(y) &= \Delta x\text{-DE}^{s_0 \to s_1}_{x_0, x_1}(y\mid x_0) - 
        \Delta x\text{-IE}^{s_0 \to s_1}_{x_1, x_0}(y\mid x_0) 
        - \Delta x\text{-SE}^{s_0 \to s_1}_{x_1, x_0}(y).
    \end{align}
    Furthermore, each of the contributions $\Delta\text{-CE} \in \{\Delta x\text{-DE}, \Delta x\text{-IE}, \Delta x\text{-SE}\}$ can be additionally decomposed as:
    \begin{equation*} \label{eq:delta-ce-decomp}
    \Delta \text{-CE}^{s_0 \to s_1}_{x_0, x_1}(y) = \Delta \text{-CE}^{s_0, s_0, s_0 \to s_1}_{x_0, x_1}(y) + \Delta \text{-CE}^{s_0, s_0 \to s_1, s_1}_{x_0, x_1}(y) 
    +\Delta \text{-CE}^{s_0 \to s_1, s_1, s_1}_{x_0, x_1}(y). \refstepcounter{equation}\tag*{$\square$ (\theequation)}
\end{equation*}
\end{theorem}
The above theorem is a key insight allowing us to understand how disparities in the generative model are structured, compared to the disparities in the real world. First, the overall disparity difference $\Delta \text{TV}_{x_0, x_1}^{s_0 \to s_1}(y)$ is decomposed into contributions from differences in direct, indirect, and spurious pathways. In particular, each $\Delta\text{-CE}^{s_0 \to s_1}$ in Eq.~\ref{eq:delta-tv-decomp} measures how much the corresponding effect changes when the real-world mechanisms $\{f^{rw}_{X,Z}, f^{rw}_W, f^{rw}_Y\}$ are replaced by the generative model's mechanisms $\{f^{gm}_{X,Z}, f^{gm}_W, f^{gm}_Y\}$. However, as the theorem shows, these quantities can be further decomposed. 
We now examine the interpretation of this decomposition, in the case of direct effects:
\begin{align} \label{eq:delta-xde-decomp}
    \Delta x\text{-DE}^{s_0 \to s_1}_{x_0, x_1}(y\mid x_0) &= \underbrace{\Delta x\text{-DE}^{s_0, s_0, s_0 \to s_1}_{x_0, x_1}(y\mid x_0)}_{T_1} + \underbrace{\Delta x\text{-DE}^{s_0, s_0 \to s_1, s_1}_{x_0, x_1}(y\mid x_0)}_{T_2} \\
    &\quad + \underbrace{\Delta x\text{-DE}^{s_0 \to s_1, s_1, s_1}_{x_0, x_1}(y\mid x_0)}_{T_3}. \nonumber
\end{align}
As mentioned, the LHS above measures the impact of exchanging all real world mechanism with the model ones. The term $T_1$ captures the change in the direct effect resulting from the transition $f_Y^{rw} \to f_Y^{gm}$ (i.e., replacing the mechanism $f_Y^{rw}$ only), while keeping the $f_{X,Z}^{rw}, f_W^{rw}$ for $X, Z, W$ variables.
The term $T_2$ captures the change in direct effect from the transition $f_W^{rw} \to f_W^{gm}$, with $f_{X,Z}^{rw}$ governing $X,Z$ behavior, and $f_Y^{gm}$ governing the $Y$ behavior, respectively. Finally, the term $T_3$ captures the change in direct effect from the transition $f_{X,Z}^{rw} \to f_{X,Z}^{gm}$, with $f_{W}^{gm}, f_Y^{gm}$ governing $W, Y$ behavior. 
Therefore, the decomposition in Eq.~\ref{eq:delta-xde-decomp} allows us to disentangle how the replacement real-world mechanism $f_{X,Z}, f_W, f_Y$ with the model ones affects the direct effect, stepwise. Using the analogous decompositions for indirect and spurious effects, we obtain a full insight into how direct, indirect, and spurious effects are impacted by the replacement of the causal mechanisms. As the following result shows, Thm.~\ref{thm:new-decomp} generalizes previous literature:
\begin{corollary}[Thm.~\ref{thm:new-decomp} in Standard ML Setting] \label{cor:standard-ml}
    Consider the machine learning setting where
    \begin{align}
        f_W^{gm} &= f_W^{rw},\; P^{gm}(U_W)=P^{rw}(U_W) ,\quad f_{X,Z}^{gm} = f_{X,Z}^{rw},\; P^{gm}(U_{X,Z})=P^{rw}(U_{X,Z})
    \end{align}
    that is, the model inherits the mechanisms for $X,Z,W$ from the real world. 
    Then, for each of the contributions $\Delta\text{-CE} \in \{\Delta x\text{-DE}, \Delta x\text{-IE}, \Delta x\text{-SE}\}$, we have that:
    \begin{align} \label{eq:delta-ce-decomp}
        \Delta \text{-CE}^{s_0, s_0 \to s_1, s_1}_{x_0, x_1}(y) &= 0, \quad
        \Delta \text{-CE}^{s_0 \to s_1, s_1, s_1}_{x_0, x_1}(y) = 0,
    \end{align}
    while only the effect $\Delta \text{-CE}^{s_0, s_0, s_0 \to s_1}_{x_0, x_1}(y)$ can be non-zero. $\hfill\square$
\end{corollary}
The above corollary considers the standard ML setting, in which the ML model inherits the mechanisms for $X,Z,W$ from the real world. This setting, visualized in Fig.~\ref{fig:sfm-ml}, is typically considered in the causal fairness literature \citep{plecko2022causal}. However, Cor.~\ref{cor:standard-ml} shows that in this ML setting, only a single term $\Delta \text{-CE}^{s_0, s_0, s_0 \to s_1}_{x_0, x_1}(y)$ will be non-zero, since only the $f_Y$ mechanism replacement occurs ($f_Y^{rw} \to f_Y^{gm}$). In the general case, needed for analyzing generative AI models, this is no longer the case -- terms in Eq.~\ref{eq:delta-ce-decomp} can be non-zero, reflecting a more involved setting in which there are contributions arising from the replacement of $f_W, f_{X,Z}$ mechanisms.
Therefore, Cor.~\ref{cor:standard-ml} highlights the fact that the new approach introduced in this paper strictly generalizes existing work.

\paragraph{Operational Aspects.} After establishing the theoretical foundations, we now turn to operational aspects, and state a key identification and estimation result (see proof and discussion in App.~\ref{appendix:proofs}): 
\begin{proposition}[ID \& Estimation] \label{prop:id-est}
    Under the assumptions encoded in the $S$-SFM, the potential outcome $\ex [Y^{s_y}_{x_y, W^{s_w}_{x_w}} \mid X = x_z, S = s_z]$ is identified as:
    \begin{align} \label{eq:po-id}
        \sum_{z, w} P^{s_y}(y \mid x_y, z, w)P^{s_w}(w \mid x_w, z) P^{s_z}(z \mid x_z),
    \end{align}
    and can be estimated efficiently via samples from $P^{s_y}(y \mid x_y, z, w)P^{s_w}(y \mid x_w, z) P^{s_z}(z \mid x_z)$, where the superscripts $s_{V_i}$ denote sampling from the environment $S = s_{V_i}$. 
    $\hfill\square$
\end{proposition}
The above proposition assumes access to samples from the distribution $P^{s_y}(y \mid x_y, z, w)P^{s_w}(y \mid x_w, z) P^{s_z}(z \mid x_z)$. Whenever we have $s_z \leq s_w \leq s_y$, such samples can be generated as follows. 
Assume we have access to the real world dataset $\mathcal{D}^{s_0}$.
If $s_z = s_w= 0, s_y =1$, we take the sample values $(X_i, Z_i, W_i)$ from $\mathcal{D}^{s_0}$, while we generate the value for $Y_i^{gm}$ from the conditional distribution $P^{s_1}(y \mid x, z, w)$ of the model.
If $s_z = 0, s_w=s_y =1$, we take the sample values $(X_i, Z_i)$ from $\mathcal{D}^{s_0}$, while we generate the values for $W_i^{gm}, Y_i^{gm}$ from the conditional distribution $P^{s_1}(y, w \mid x, z)$ of the model. If $s_z = s_w = s_y = 1$, all values are generated from $P^{s_1}(x,z,w,y)$. 

As noted earlier, generative models may operate at a different level of granularity, generating data in lower-level modalities. In this case, we assume access to a generator $\Gamma$, which may take arbitrary inputs of the form $V \mid V' = v'$, and produce outputs that imply the values of covariates $V$ conditional on the values $V' = v'$. This output is then passed into an annotator $A$, which extracts the values of $V$ from the lower-level data (full details are described in App.~\ref{appendix:elicit}). 
To ground the discussion, we provide the instantiation of this workflow for language models:
\begin{example}[Bias Quantification in LLMs] \label{ex:nsduh-marijuana}
    Consider quantification of racial bias in substance abuse for large language models. We wish to study how the protected attribute $X$ (race) affects the outcome $Y$ (smoking marijuana monthly) via direct, indirect (mediators education $W_1$, income $W_2$), and spurious (confounders sex $Z_1$, age $Z_2$) mechanism, both in the real world and the language model. For real world data $\mathcal{D}^{s_0}$, we use the NSDUH dataset \citep{nsduh2023}.

    When generating the data from the distribution $P^{s_0}(x, z)P^{s_1}(y, w \mid x,z)$, we take the values of $(X_i, Z_i)$ from the dataset $\mathcal{D}^{s_0}$, and use the following prompt template that is passed to an LLM 
    \begin{align} \nonumber
        \pi(X, Z_1, Z_2) =& \text{``For a \{$X$\} person, of \{$Z_1$\} sex, age \{$Z_2$\}, write a story mentioning their}\\
        &\text{education, income, and whether they smoke marijuana every month.''} \nonumber
    \end{align}
    For specific values $X=$White, $Z_1$=female, $Z_2=24$, the LLM is passed the following prompt:
    \vspace{-0.05in}
    \begin{verbatim}
    For a White person, of female sex, age 24, write a story mentioning 
    their education, income, and whether they smoke marijuana every month.
    \end{verbatim}
    \vspace{-0.2in}
    The LLM queried with such prompts can be viewed as the generator $\Gamma$, taking queries of the form $Y, W \mid X=x, Z=z$ as inputs. $\Gamma$ may generate a response as follows:
    \vspace{-0.05in}
\begin{verbatim}
    Sophia is a 24-year-old White woman with a bachelor’s degree in
    environmental science. She earns $30,000–$39,999 a year from her
    part-time job at a local outdoor-gear shop. She smokes marijuana
    monthly with friends, saying it’s her way to relax.
\end{verbatim}
    Finally, responses from $\Gamma$ are passed to an annotator $A$, another LLM, which then determines the values of $W, Y$. In this case $W_1=\text{Bachelor's}, W_2=\text{\$30,000-\$39,999}, Y=1$. 
    The sample $(X_i, Z_i, W_i^{gm}, Y_i^{gm})$ is added to 
    the dataset 
    $\mathcal{D}^{s_0, s_1, s_1}$, which is drawn from the target distribution $P^{s_0}(x, z)P^{s_1}(y, w\mid x,z)$.
\end{example}
The example illustrates how to generate the dataset $\mathcal{D}^{s_0, s_1, s_1}$. Generally, we focus on four datasets, $\mathcal{D}^{s_0}$ (real world data), $\mathcal{D}^{s_0, s_0, s_1}$ where the $f_{Y}$ mechanism is taken from the generative model, $\mathcal{D}^{s_0, s_1, s_1}$ where $f_W, f_Y$ are taken from the generative model, and $\mathcal{D}^{s_1}$, where all data is fully sampled from the generative model. 
Finally, for each of the four datasets, the result from Prop.~\ref{prop:id-est} can be used to estimate the required causal effects of interest (see App.~\ref{appendix:proofs} for details).

We also remark that other datasets $\mathcal{D}^{s_z,s_w,s_y}$ could be considered, conceptually. For instance, $\mathcal{D}^{s_1,s_0, s_0}$ would correspond to data where $X, Z$ variables are drawn from the generative model, and then passed into real-world mechanisms for $W, Y$. However, in such cases the conditional distribution $P(Y, W \mid X, Z)$ would need to be learned using real data, for a possibly multi-dimensional $W$, which may be a difficult exercise with a limited amount of data. However, when $s_z \leq s_w \leq s_y$, the dataset $\mathcal{D}^{s_z,s_w,s_y}$ can be generated without fitting complex conditionals on real data. 
\section{Experiments}\label{sec:experiments}
\paragraph{Datasets.} We use three publicly available datasets, each providing nationally representative individual-level data with sample weights:
(i) \textbf{NSDUH 2023}~\citep{nsduh2023} from SAMHSA, where we examine racial bias in monthly marijuana use among White ($X = x_0$) vs. minority ($X=x_1$ for African-American or Hispanic) individuals ($Y=1$ for smoking marijuana);
(ii) \textbf{BRFSS 2023}~\citep{brfss2023} from the CDC, where we examine racial bias in diabetes diagnosis among the same groups ($Y=1$ for diabetes);
(iii) \textbf{ACS 2023}~\citep{acs2023} from the US Census Bureau, where we examine sex bias ($X = x_0$ male, $X = x_1$ female) in salary among employed adults ($Y=1$ for income below \$50k/year).
For each dataset, the protected attribute $X$, confounders $Z$, mediators $W$, and outcome $Y$ are specified by the SFMs in Fig.~\ref{fig:sfms}. We sample $n = 8{,}192$ rows per dataset using survey weights, and use the same sampled rows as input across all models to ensure comparability.

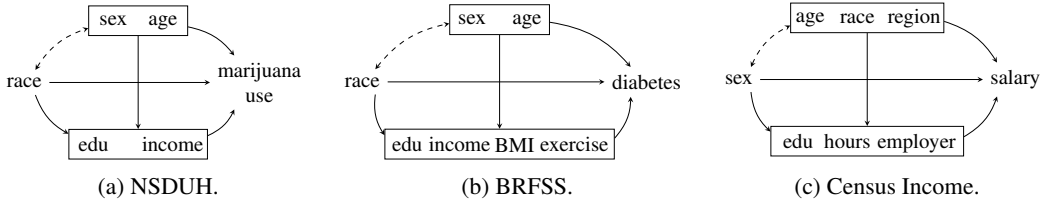
\begin{figure}[t]
\centering
\begin{subfigure}{0.3\textwidth}
\centering
\resizebox{\linewidth}{!}{%
\begin{tikzpicture}
	 	[>=stealth, rv/.style={thick}, rvc/.style={triangle, draw, thick, minimum size=8mm}, node distance=7mm]
	 	\pgfsetarrows{latex-latex};
	 	\large
        \begin{scope}
	 	\node[rv] (c) at (1.8,1.25) {sex};
        \node[rv] (c) at (2.8,1.22) {age};
	 	\node[rv] (a) at (0,0) {race};
	 	\node[rv, align=center] (m) at (1.4,-1.25) {edu};
	 	\node[rv, align=center] (l) at (3,-1.25) {income};	 	
        \node[rv, align=center] (y) at (4.75,0) {marijuana\\ use};
	 	
	 	\node (Zset) [draw,rectangle,minimum width=2cm,minimum height=0.6cm] at (2.3,1.25) {};
	 	\node (Wset) [draw,rectangle,minimum width=2.8cm,minimum height=0.6cm] at (2.3,-1.25) {};
	 	
	 	\path[->] (a) edge[bend left = 0] (y);
	 	\path[->] (a) edge[bend left = -20] (Wset);
	 	\path[->] (Zset) edge[bend left = 0] (Wset);
	 	\path[->] (Wset) edge[bend left = -20] (y);
	 	\path[->] (Zset) edge[bend left = 20] (y);
	 	
	 	\path[<->,dashed] (a) edge[bend left = 20](Zset);
	 	\end{scope}
\end{tikzpicture}
}
\caption{NSDUH.}
\label{fig:sfm-nsduh}
\end{subfigure}
\hfill
\begin{subfigure}{0.34\textwidth}
\centering
\resizebox{\linewidth}{!}{%
\begin{tikzpicture}
    [>=stealth, rv/.style={thick}, node distance=7mm]
    \pgfsetarrows{latex-latex};
    \large
    \begin{scope}
    \node[rv] at (2.0,1.25) {sex};
    \node[rv] at (3.1,1.22) {age};
    \node[rv] (a) at (-0.2,0) {race};
    \node[rv] at (0.7,-1.25) {edu};
    \node[rv] at (1.75,-1.25) {income};
    \node[rv] at (2.85,-1.275) {BMI};
    \node[rv] at (4.05,-1.25) {exercise};
    \node[rv] (y) at (5.5,0) {diabetes};
    \node (Zset) [draw,rectangle,minimum width=2cm,minimum height=0.6cm] at (2.55,1.25) {};
    \node (Wset) [draw,rectangle,minimum width=4.6cm,minimum height=0.6cm] at (2.55,-1.25) {};
    \path[->] (a) edge (y);
    \path[->] (a) edge[bend left = -20] (Wset);
    \path[->] (Zset) edge (Wset);
    \path[->] (Wset) edge[bend left = -20] (y);
    \path[->] (Zset) edge[bend left = 20] (y);
    \path[<->,dashed] (a) edge[bend left = 20] (Zset);
    \end{scope}
\end{tikzpicture}
}
\caption{BRFSS.}
\label{fig:sfm-brfss}
\end{subfigure}
\hfill
\begin{subfigure}{0.32\textwidth}
\centering
\resizebox{\linewidth}{!}{%
\begin{tikzpicture}
    [>=stealth, rv/.style={thick}, node distance=7mm]
    \pgfsetarrows{latex-latex};
    \large
    \begin{scope}
    \node[rv] at (1.45,1.22) {age};
    \node[rv] at (2.4,1.25) {race};
    \node[rv] at (3.55,1.25) {region};
    \node[rv] (a) at (0,0) {sex};
    \node[rv] at (1.2,-1.25) {edu};
    \node[rv] at (2.2,-1.25) {hours};
    \node[rv] at (3.6,-1.29) {employer};
    \node[rv] (y) at (5.6,0) {salary};
    \node (Zset) [draw,rectangle,minimum width=3.1cm,minimum height=0.6cm] at (2.6,1.25) {};
    \node (Wset) [draw,rectangle,minimum width=3.9cm,minimum height=0.6cm] at (2.6,-1.25) {};
    \path[->] (a) edge (y);
    \path[->] (a) edge[bend left = -20] (Wset);
    \path[->] (Zset) edge (Wset);
    \path[->] (Wset) edge[bend left = -20] (y);
    \path[->] (Zset) edge[bend left = 20] (y);
    \path[<->,dashed] (a) edge[bend left = 20] (Zset);
    \end{scope}
\end{tikzpicture}
}
\caption{Census Income.}
\label{fig:sfm-census}
\end{subfigure}
\caption{Standard Fairness Models for the three datasets.}
\vspace{-0.2in}
\label{fig:sfms}
\end{figure}

\paragraph{Models.} We evaluate 10 instruction-tuned open-weight models spanning 6 families and varying parameter count: Llama 3.1 8B and Llama 3.3 70B~\citep{grattafiori2024llama}, Qwen 3.5 9B and Qwen 3.5 27B~\citep{qwen35}, Gemma 3 4B and Gemma 3 27B~\citep{team2024gemma}, DeepSeek 7B~\citep{bi2024deepseek} and DeepSeek-R1 32B (Qwen-distilled)~\citep{guo2025deepseek}, Ministral 3 8B~\citep{ministral3}, and Phi-4~\citep{abdin2024phi}. All generation is performed at temperature $= 1$ and $\text{top-}p = 1$ to mimic the model's real world deployment. 
Following Sec.~\ref{sec:cfgai}, we construct three intermediate datasets per model, labeled $\mathcal{D}^{s_0,s_0,s_1}$, $\mathcal{D}^{s_0,s_1,s_1}$, and $\mathcal{D}^{s_1}$, corresponding to progressive replacement of mechanisms $f_Y$, $f_W$, and $f_{X,Z}$. 
During prompting, models are provided with the dataset year, country, and cohort definition (e.g., all adults).
App.~\ref{appendix:elicit} provides details of dataset generation (including annotator validation in App.\ref{appendix:annotator-validation}), while App.~\ref{appendix:empirical} extends empirical analyses.

\subsection{Global Analysis}\label{sec:global}

Across our 10 models and 3 datasets, the methodology of Sec.~\ref{sec:cfgai} produces, for each model-dataset pair, three causal effects ($x$-DE, $x$-IE, $x$-SE) at three intermediate stages of mechanism replacement (replacing $f_Y$, then $f_W$, then $f_{X,Z}$), yielding the models 9-dimensional \emph{bias signature} $\mathcal{B}(M, \mathcal{D})$:
\begin{align}
    (\text{DE}^{s_0,s_0, s_1}, \text{IE}^{s_0,s_0, s_1}, \text{SE}^{s_0,s_0, s_1},
    \text{DE}^{s_0,s_1, s_1}, \text{IE}^{s_0,s_1, s_1}, \text{SE}^{s_0,s_1, s_1},
    \text{DE}^{s_1}, \text{IE}^{s_1}, \text{SE}^{s_1}),
\end{align}
per model-dataset pair, which we then compare against the corresponding real-world effects. For ease of interpretation, we report $\text{IE} =-x\text{-IE}_{x_1, x_0}(y \mid x_0)$ and $\text{SE} =-x\text{-SE}_{x_1, x_0}(y)$, as these consider a reverse transition compared to $\text{DE}=x\text{-DE}_{x_0,x_1}(y\mid x_0)$.

\paragraph{Disadvantage and Advantage.}
For a given causal pathway, we say that the real-world effect places the protected group at a \emph{disadvantage} when its sign indicates a worse outcome for $X = x_1$ relative to $X = x_0$ (e.g., higher probability of marijuana use, higher probability of diabetes, or lower salary).
Formally, this happens if $\text{CE} > 0$.
Conversely, the effect indicates an \emph{advantage} when its sign points in the opposite direction. Most pathways across our datasets show protected group disadvantage; the exceptions are NSDUH direct effect (minorities use less marijuana) and BRFSS spurious effect (minorities in the cohort are younger, and lower age reduces diabetes risk).

Given a causal pathway with a real-world effect $\text{CE}^{s_0}$, we classify how the model's corresponding effect $\text{CE}^{s_z, s_w, s_y}$ (of any intermediate stage) modifies the real-world bias. 
The model \emph{amplifies} the bias when $|\text{CE}^{s_1}| > |\text{CE}^{s_0}|$ with matching sign, \emph{dampens} it when $|\text{CE}^{s_1}| < |\text{CE}^{s_0}|$ with matching sign, and \emph{reverses} it when the sign flips. Each of the $9$ effects per model-dataset pair therefore receives one of these three labels, and we summarize across all $3 \times 9 = 27$ effects per model.
Tab.~\ref{tab:stereotype-all} reports, for each model, the proportion of disadvantage-direction effects falling into each category.
\emph{Reversal}, where the model induces an effect of opposite sign to reality, occurs in 7--44\% of disadvantage-direction effects depending on the model. \emph{Amplification} occurs for 15--44\% of the effects. No model dampens consistently; even the largest model in our suite (Llama 3.3 70B) modifies the majority of effects substantially. 
Regarding advantage directions: nearly all models dampen or reverse the BRFSS spurious advantage, while models exhibit mixed behavior with respect to NSDUH direct advantage (see Gemma 3 27B case study below).

\begin{figure}[t]
\centering
\begin{minipage}{0.45\textwidth}
\centering
\centering
\begin{tabular}{l ccc}
\toprule
\textbf{Model} & Amp. & Damp. & Rev. \\
\midrule
DeepSeek 7B & \cellcolor{orange!10}15\% & \cellcolor{orange!65}56\% & \cellcolor{orange!35}30\% \\
DeepSeek-R1 32B & \cellcolor{orange!35}30\% & \cellcolor{orange!65}59\% & \cellcolor{orange!10}11\% \\
Gemma 3 27B & \cellcolor{orange!35}44\% & \cellcolor{orange!35}30\% & \cellcolor{orange!35}26\% \\
Gemma 3 4B & \cellcolor{orange!35}41\% & \cellcolor{orange!35}33\% & \cellcolor{orange!35}26\% \\
Llama 3.3 70B & \cellcolor{orange!35}44\% & \cellcolor{orange!35}44\% & \cellcolor{orange!10}11\% \\
Llama 3 8B & \cellcolor{orange!35}41\% & \cellcolor{orange!35}44\% & \cellcolor{orange!10}15\% \\
Ministral 3 8B & \cellcolor{orange!35}37\% & \cellcolor{orange!65}56\% & \cellcolor{orange!10}7\% \\
Phi-4 & \cellcolor{orange!10}22\% & \cellcolor{orange!65}56\% & \cellcolor{orange!10}22\% \\
Qwen 3.5 27B & \cellcolor{orange!10}15\% & \cellcolor{orange!35}41\% & \cellcolor{orange!35}44\% \\
Qwen 3.5 9B & \cellcolor{orange!35}44\% & \cellcolor{orange!35}48\% & \cellcolor{orange!10}7\% \\
\bottomrule
\end{tabular}
\captionof{table}{Stereotype analysis summary (all datasets).}
\label{tab:stereotype-all}

\end{minipage}
\hfill
\begin{minipage}{0.5\textwidth}
\centering
\includegraphics[width=\linewidth]{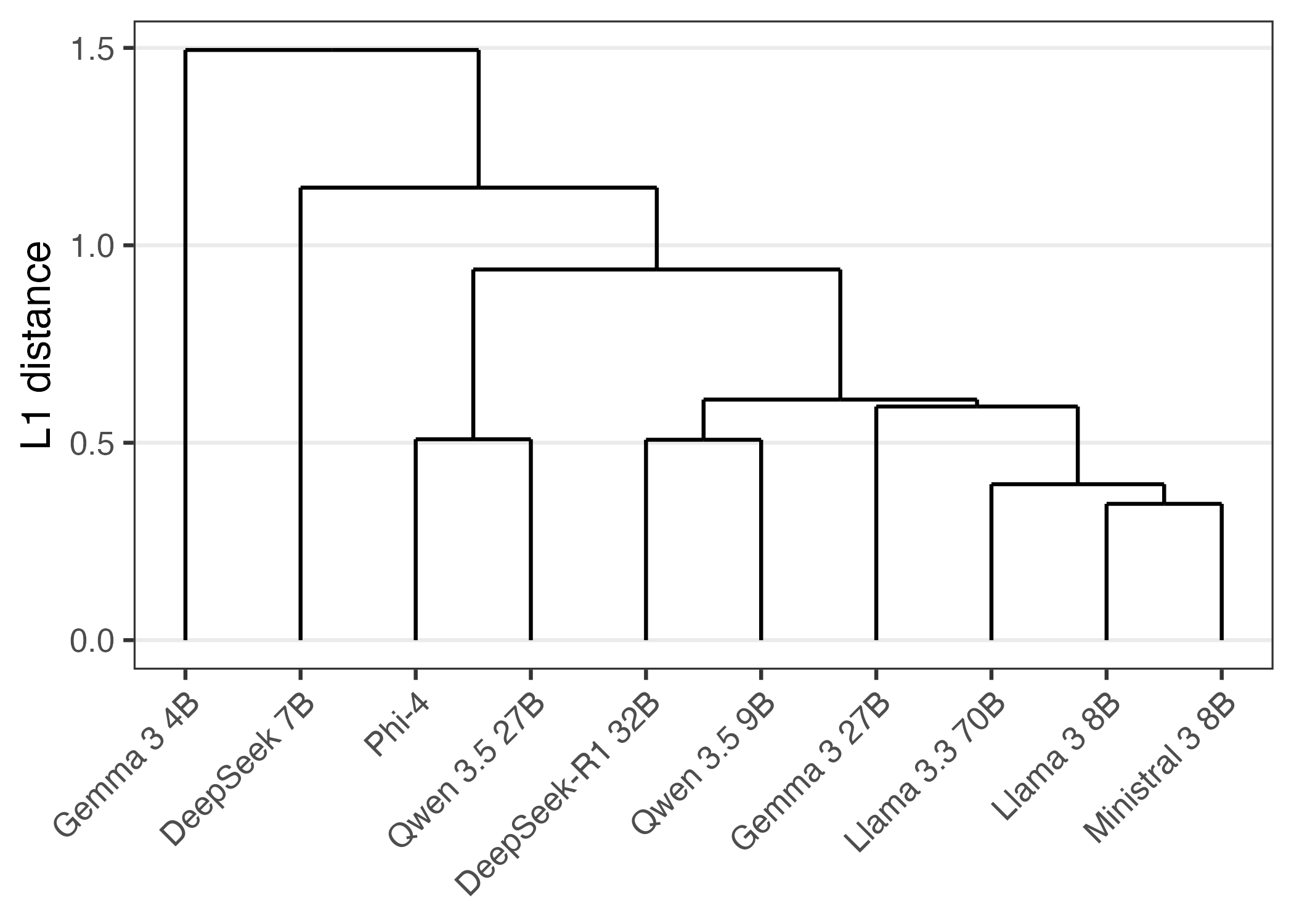}
\vspace{-0.3in}
\captionof{figure}{Hierarchical clustering of model bias signatures ($L_1$ distance, Ward linkage).}
\label{fig:dendrogram}
\end{minipage}
\end{figure}


\paragraph{Bias Similarity and Model Families.}
Concatenating $\mathcal{B}(M, \mathcal{D})$ across datasets defines a 27-dimensional bias signature for each model. To assess structural similarities in encoded causal beliefs, we compute pairwise $L_1$ distances and visualize the results via hierarchical clustering (Fig.~\ref{fig:dendrogram}, Ward linkage). The resulting dendrogram suggests that bias similarity is not highest within model families: the closest pair overall are Llama 3 8B and Ministral3 8B ($L_1 = 0.35$), a pairing among distinct developer families and parameter scales. 
Further, while the Llama 3 siblings sit at $L_1 = 0.39$, the Qwen 3.5 pair is at $0.62$ and the Gemma 3 pair at $1.22$ -- both farther apart than many cross-family pairs. These mixed groupings motivate a formal test of whether family membership reliably predicts bias similarity. Using a permutation test ($N = 5{,}000$) on the mean within-family $L_1$ distance, we observe a mean distance of $0.80$ against a null mean of $0.76$ ($p = 0.62$) -- not offering evidence that siblings are closer than chance (however, we remark that the statistical power of our test cannot be computed).
We close with two case studies showcasing granular insights enabled by our framework.

\paragraph{Gemma 3 27B Thinks Minorities Use More Marijuana.}
In reality, NSDUH data shows that minority individuals use marijuana at a \emph{lower} rate than the majority along the direct pathway, controlling for sex, age, education, and income: $\text{DE}^{s_0} = -3.8\% \pm 1.8\%$ (all effects are reported with $\pm 1.96\sigma$, the 95\% confidence interval, constructed assuming asymptotic normality; see App.~\ref{appendix:proofs}). The indirect and spurious paths point in the opposite direction with smaller magnitudes ($\text{IE}^{s_0} = 0.2\% \pm 0.7\%$, $\text{SE}^{s_0} = 1.8\% \pm 1.0\%$). Replacing both the outcome ($f_Y$) and mediator ($f_W$) mechanisms with Gemma 3 27B's (dataset $\mathcal{D}^{s_0, s_1, s_1}$) reverses the sign of the direct effect to a significant $+4.3\% \pm 2.9\%$ and amplifies the spurious effect to $+6.7\% \pm 3.2\%$, while the indirect path remains negligible ($-0.4\% \pm 1.8\%$). The model encodes a stereotyped belief that minority individuals use marijuana at \emph{higher} rates than the majority, even after controlling for socioeconomic factors, which directly contradicts the empirical pattern in the data. Such bias may propagate into downstream tasks (e.g., risk scoring, policy recommendations) where stereotypes are damaging.

\paragraph{Qwen 3.5 27B Shows Major Sign Reversal for BRFSS.}
In reality, minority respondents in BRFSS face elevated risk of diabetes along the direct pathway ($\text{DE}^{s_0} = 5.7\% \pm 2.0\%$) and along the indirect pathway through differences in education, income, BMI, and exercise ($\text{IE}^{s_0} = 2.7\% \pm 1.3\%$), but a reduced spurious risk due to the cohort's younger age ($\text{SE}^{s_0} = -5.2\% \pm 1.3\%$). Under full mechanism replacement (dataset $\mathcal{D}^{s_1}$), Qwen 3.5 27B reverses the sign of both the direct and indirect effects, and dampens the spurious effect toward zero: $\text{DE}^{s_1} = -1.8\% \pm 1.6\%$, $\text{IE}^{s_1} = -6.3\% \pm 1.5\%$, $\text{SE}^{s_1} = -1.6\% \pm 3.3\%$. The model believes the majority population is at increased risk of diabetes via direct and indirect mechanisms, contrary to the epidemiological pattern, while no longer encoding minority cohort age advantage. Both the $\text{DE}$ and $\text{IE}$ reversals are statistically significant; the dampening of $\text{SE}$ leaves it not significantly different from zero.
\paragraph{Limitations and Future Work.}
First, generative models are capable of making inferences in the opposite causal order (e.g., inferring sex from salary information), a setting our framework does not cover explicitly, and which we leave for future work.
Second, our framework assumes the ability to query generative models for arbitrary conditionals, specifying any value of the observed covariates $V$. While reasonable in language and multimodal settings, it may not hold in others; in vision models, for instance, it may be difficult to convey a person's education level through an image. Such cases may require further methodological advances.
Third, applying our framework to a new domain requires specifying the $S$-SFM, which requires human elicitation of causal assumptions. Relaxing this requirement is left for future work. Finally, we restricted our study to open-weight models; extension to closed models \citep{team2023gemini, achiam2023gpt, bai2022constitutional} is methodologically trivial but cost-prohibitive at scale.

\newpage
\ifnum\treportflag=0
\section*{Broader Impact Statement}

This work analyzes fairness in generative AI systems from a causal lens, and we not forsee obvious risks arising from this work. 
In fact, the work is intended to contribute positively to the study and mitigation of fairness concerns. 
The developed methodology represents an auditing tool: it enables practitioners, regulators, and model developers to quantify, in a causally interpretable manner, how the beliefs encoded by a generative model differ from the relationships observed in real-world data, and to trace those differences to specific mechanisms.

However, we emphasize that our framework diagnoses bias but does not offer ways to correct for such bias. Further, the conclusions based on our framework depend on the underlying $S$-SFM, which encodes causal assumptions about the underlying domain. 
Different reasonable choices of graph structure may yield different conclusions, and our results should therefore be interpreted in light of the assumptions they rest on (rather than being treated as definitive). 

All datasets used in this paper (NSDUH, BRFSS, ACS) are publicly available, de-identified survey data released by US government agencies (SAMHSA, CDC, U.S.\ Census Bureau, respectively).
\fi

\bibliographystyle{abbrvnat}
\bibliography{refs}


\newpage
\appendix
\section*{\centering\Large Supplementary Material for \textit{\ttl{}}}
The source code for reproducing the results can be found in the anonymized code repository \url{https://anonymous.4open.science/r/fgai-48B7}. The repository includes a README file explaining the setup steps.
All experiments were run on four NVIDIA A100 GPUs with 40GB VRAM each. Inference was performed using vLLM \citep{kwon2023efficient}, and the total computation time required to generate the results across all models and datasets is under 72 hours.

\section{Proofs} \label{appendix:proofs}
In this appendix, we provide the proofs for the key results, including Thm.~\ref{thm:new-decomp}, Cor.~\ref{cor:standard-ml}, and Prop.~\ref{prop:id-est}. After establishing the identification of the $S$-modified potential outcomes, we also discuss the estimation of such quantities.
Recall that in the $S$-SFM (Fig.~\ref{fig:sfm-genai}), the node $S$ is a root node with no parents, pointing to each of the covariates $X, Z, W, Y$. This structural property allows us to freely exchange conditioning on $S$ and intervention on $S$: for any subset of covariates, conditioning on $\{S = s\}$ and intervening to set $S = s$ induce the same distribution. Using this, we define the following $L_1$ shorthand notation:
\begin{align}
    P^{s_z}(z \mid x_z) &\overset{\Delta}{=} P(z \mid x_z, S = s_z), \\
    P^{s_w}(w \mid x_w, z) &\overset{\Delta}{=} P(w \mid x_w, z, S = s_w), \\
    P^{s_y}(y \mid x_y, z, w) &\overset{\Delta}{=} P(y \mid x_y, z, w, S = s_y),
\end{align}
with analogous notation for expectations. With this convention, the identification expression in Eq.~\ref{eq:po-id} is a function of observable $L_1$ conditionals across the real-world and generative-model environments.

\begin{proof}[Proof of Thm.~\ref{thm:new-decomp}]
    Note that we have for $s \in \{ s_0, s_1 \}$
    \begin{align}
        \text{TV}_{x_0, x_1}^s(y) &= \ex^{s}[Y \mid X = x_1] - \ex^{s}[Y \mid X = x_0] \\
        &= \ex^{s}[Y_{x_1} \mid X = x_1] - \ex^{s}[Y_{x_0} \mid X = x_0] \\
        &= \ex^{s}[Y_{x_1} \mid X = x_1] - \ex^{s}[Y_{x_1} \mid X = x_0] \\
        &\quad + \ex^{s}[Y_{x_1} \mid X = x_0] - \ex^{s}[Y_{x_1, W_{x_0}} \mid X = x_0] \\
        &\quad + \ex^{s}[Y_{x_1, W_{x_0}} \mid X = x_1] - \ex^{s}[Y_{x_0} \mid X = x_0] \\
        &= - x\text{-SE}^{s}_{x_1, x_0}(y) - x\text{-IE}^{s}_{x_1, x_0}(y\mid x_0) + x\text{-DE}^{s}_{x_0, x_1}(y \mid x_0), \label{eq:tv-decomp}
    \end{align}
    where the first step follows from the consistency axiom, second by addition-subtraction, and the third by recognizing the appropriate causal effects. Using this Eq.~\ref{eq:tv-decomp}, we can re-write $\Delta \text{TV}_{x_0, x_1}^{s_0 \to s_1}(y)$ as
    \begin{align}
        \Delta \text{TV}_{x_0, x_1}^{s_0 \to s_1}(y) &\overset{\Delta}{=} \text{TV}^{gm}_{x_0, x_1} - \text{TV}^{rw}_{x_0, x_1} = \text{TV}^{s_1}_{x_0, x_1} - \text{TV}^{s_0}_{x_0, x_1} \\
        &= x\text{-DE}^{s_1}_{x_0, x_1}(y\mid x_0) - x\text{-IE}^{s_1}_{x_1, x_0}(y\mid x_0) - x\text{-SE}^{s_1}_{x_1, x_0}(y) \label{eq:tv-decomp-gm} \\
        &\quad - (x\text{-DE}^{s_0}_{x_0, x_1}(y\mid x_0) - x\text{-IE}^{s_0}_{x_1, x_0}(y\mid x_0) - x\text{-SE}^{s_0}_{x_1, x_0}(y)) \nonumber \\
        &= x\text{-DE}^{s_1}_{x_0, x_1}(y\mid x_0) - x\text{-DE}^{s_0}_{x_0, x_1}(y\mid x_0) \label{eq:x-ce-group} \\
        &\quad - (x\text{-IE}^{s_1}_{x_1, x_0}(y\mid x_0) - x\text{-rw}^{s_0}_{x_1, x_0}(y\mid x_0)) \nonumber \\
        &\quad - (x\text{-SE}^{s_1}_{x_1, x_0}(y) - x\text{-SE}^{s_0}_{x_1, x_0}(y)) \nonumber \\
        &= \Delta x\text{-DE}^{s_0 \to s_1}_{x_0, x_1}(y\mid x_0) - 
        \Delta x\text{-IE}^{s_0 \to s_1}_{x_1, x_0}(y\mid x_0) 
        - \Delta x\text{-SE}^{s_0 \to s_1}_{x_1, x_0}(y) \label{eq:delta-tv-target}
    \end{align}
    with Eq.~\ref{eq:tv-decomp-gm} following by an application of Eq.~\ref{eq:tv-decomp}, Eq.~\ref{eq:x-ce-group} by rearrangement, and Eq.~\ref{eq:delta-tv-target} by definition of $\Delta x\text{-}CE$ effects. Now, note that
    \begin{align}
        \Delta x\text{-DE}^{s_0 \to s_1}_{x_0, x_1}(y\mid x_0) &\overset{\Delta}{=} x\text{-DE}^{s_1, s_1, s_1}_{x_0, x_1}(y\mid x_0) - x\text{-DE}^{s_0, s_0, s_0}_{x_0, x_1}(y\mid x_0) \label{eq:delta-x-de-def} \\
        &= x\text{-DE}^{s_1, s_1, s_1}_{x_0, x_1}(y\mid x_0) - x\text{-DE}^{s_0, s_1, s_1}_{x_0, x_1}(y\mid x_0)  \label{eq:delta-rearrange}\\
        &\quad + x\text{-DE}^{s_0, s_1, s_1}_{x_0, x_1}(y\mid x_0) - x\text{-DE}^{s_0, s_0, s_1}_{x_0, x_1}(y\mid x_0) \nonumber  \\
        &\quad + x\text{-DE}^{s_0, s_0, s_1}_{x_0, x_1}(y\mid x_0) - x\text{-DE}^{s_0, s_0, s_0}_{x_0, x_1}(y\mid x_0) \nonumber \\
        &= \Delta x\text{-DE}^{s_0 \to s_1, s_1, s_1}_{x_0, x_1}(y \mid x_0) + \Delta x \text{-DE}^{s_0, s_0 \to s_1, s_1}_{x_0, x_1}(y) \label{eq:delta-decomp-target} \\
    &\quad +\Delta x\text{-DE}^{s_0, s_0, s_0 \to s_1}_{x_0, x_1}, \nonumber
    \end{align}
    where Eq.~\ref{eq:delta-x-de-def} follows by definition, Eq.~\ref{eq:delta-rearrange} by addition-subtraction, and Eq.~\ref{eq:delta-decomp-target} by recognizing the corresponding $\Delta x\text{-DE}$ effects. The proof for $\Delta x\text{-IE}$ and $\Delta x\text{-SE}$ effects follows analogously.
\end{proof}

\begin{proof}[Proof of Cor.~\ref{cor:standard-ml}]
    Denote by $ml$ the generative model of the standard machine learning setting, for which $f_{X,Z}^{ml} = f_{X,Z}^{rw},\;f_W^{ml} = f_W^{rw},\; P^{ml}(U_{XZ}) = P^{rw}(U_{XZ}), \; P^{ml}(U_{W}) = P^{rw}(U_{W})$ (in words, mechanisms and noise distributions for $X, Z$, and $W$ are inherited from the real world). Then, focusing on the direct effect, we note that 
    \begin{align}
        \Delta x\text{-DE}^{s_0 \to s_1, s_1, s_1}_{x_0, x_1}(y \mid x_0) &= x\text{-DE}^{s_1, s_1, s_1}_{x_0, x_1}(y\mid x_0) - x\text{-DE}^{s_0, s_1, s_1}_{x_0, x_1}(y\mid x_0) \\
        &= \ex [Y^{s_1}_{x_1, W_{x_0}^{s_1}} - Y^{s_1}_{x_0, W_{x_0}^{s_1}} \mid X = x_0, S=s_1] \\
        &\quad - \ex [Y^{s_1}_{x_1, W_{x_0}^{s_1}} - Y^{s_1}_{x_0, W_{x_0}^{s_1}} \mid X=x_0, S=s_0].
    \end{align}
    Now, since $f_{X,Z}^{ml} = f_{X,Z}^{rw}$ and $P^{ml}(U_{XZ}) = P^{rw}(U_{XZ})$, conditioning on $X = x_0, S={s_1}$ and $X = x_0, S = {s_0}$ corresponds to conditioning on the same event; and thus 
    \begin{align}
        \ex [Y^{s_1}_{x_1, W_{x_0}^{s_1}} - Y^{s_1}_{x_0, W_{x_0}^{s_1}} \mid X= x_0, S ={s_1} ] = \ex [Y^{s_1}_{x_1, W_{x_0}^{s_1}} - Y^{s_1}_{x_0, W_{x_0}^{s_1}} \mid X = x_0, S={s_0}],
    \end{align}
    from which it follows that $\Delta x\text{-DE}^{s_0 \to s_1, s_1, s_1}_{x_0, x_1}(y \mid x_0) = 0$. Now, note further that 
    \begin{align}
        \Delta x\text{-DE}^{s_0, s_0 \to s_1, s_1}_{x_0, x_1}(y \mid x_0) &= x\text{-DE}^{s_0, s_1, s_1}_{x_0, x_1}(y\mid x_0) - x\text{-DE}^{s_0, s_0, s_1}_{x_0, x_1}(y\mid x_0) \\
        &= \ex [Y^{s_1}_{x_1, W_{x_0}^{s_1}} - Y^{s_1}_{x_0, W_{x_0}^{s_1}} \mid X = x_0, S = {s_0}] \\
        &\quad - \ex [Y^{s_1}_{x_1, W_{x_0}^{s_0}} - Y^{s_1}_{x_0, W_{x_0}^{s_0}} \mid X= x_0, S={s_0}]
    \end{align}
    Now, since $f_{W}^{ml} = f_{W}^{rw}$, we know that $W_{x_0}^{s_1}(u) = W_{x_0}^{s_0}(u)$ for all $u$, and since $P^{ml}(U_{W}) = P^{rw}(U_{W})$, we have that $W_{x_0}^{s_1} \overset{d}{=} W_{x_0}^{s_0}$, from which it follows that $Y^{s_1}_{x, W_{x_0}^{s_1}} \overset{d}{=} Y^{s_1}_{x, W_{x_0}^{s_1}}$ conditional on any event, for any value $x \in \{x_0, x_1\}$. It thus follows that
    \begin{align}
        \ex [Y^{s_1}_{x_1, W_{x_0}^{s_1}} - Y^{s_1}_{x_0, W_{x_0}^{s_1}} \mid X = x_0, S = {s_0} ] =\ex [Y^{s_1}_{x_1, W_{x_0}^{s_0}} - Y^{s_1}_{x_0, W_{x_0}^{s_0}} \mid X = x_0, S = {s_0}], 
    \end{align}
    from which it follows that $\Delta x\text{-DE}^{s_0, s_0 \to s_1, s_1}_{x_0, x_1}(y \mid x_0) = 0$. 
\end{proof}

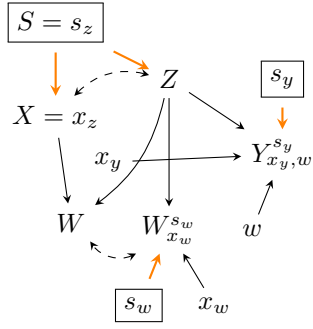
\begin{figure}
    \centering
    \scalebox{1}{\begin{tikzpicture}
 [>=stealth, rv/.style={thick}, rvc/.style={triangle, draw, thick, minimum size=4mm}, node distance=10mm]
 \pgfsetarrows{latex-latex};

 \node[rv] (z) at (0,1) {$Z$};

 \node[rv] (xw) at (0.6,-2) {$x_w$};
 \node[rv] (xz) at (-1.5,0.5) {$X=x_z$};
 \node[rv] (xy) at (-0.8,-0.1) {$x_y$};

 \node[rv] (sz) at (-1.5, 1.75) {\fbox{$S=s_z$}};
 \node[rv] (sw) at (-0.4, -2) {\fbox{$s_w$}};
 \node[rv] (sy) at (1.5, 1) {\fbox{$s_y$}};

 \node[rv] (wnat) at (-1.3,-0.9) {$W$};

 \node[rv] (w) at (0,-1) {$W_{x_w}^{s_w}$};

 \node[rv] (wfix) at (1.1,-1) {$w$};
 \node[rv] (y) at (1.5,0) {$Y_{x_y, w}^{s_y}$};

 \draw[->] (xw) -- (w);
 \draw[->] (xy) -- (y);
 \path[<->] (xz) edge[bend left=30, dashed] (z);
  \path[<->] (w) edge[bend left=40, dashed] (wnat);

 \draw[->] (z) -- (y);
 \draw[->] (z) -- (w);
 \draw[->] (z) edge[bend left=20] (wnat);
 \draw[->] (xz) -- (wnat);
 \draw[->] (wfix) -- (y);

 \draw[->,thick, orange] (sz) -- (xz);
 \draw[->,thick, orange] (sz) -- (z);
 \draw[->,thick, orange] (sw) -- (w);
 \draw[->,thick, orange] (sy) -- (y);

\end{tikzpicture}}
    \caption{Counterfactual graph for proof of Prop.~\ref{prop:id-est}.}
    \label{fig:po-ctf-graph}
\end{figure}
\begin{proof}[Proof Prop.~\ref{prop:id-est}]
Let $\mathcal{M}$ be an SCM compatible with the $S$-SFM in Fig.~\ref{fig:sfm-genai}. We wish to show that:
\begin{align}
    \ex\left[Y^{s_y}_{x_y, W^{s_w}_{x_w}} \mid X = x_z, S = s_z\right] = \sum_{z, w} P^{s_y}(y \mid x_y, z, w) \, P^{s_w}(w \mid x_w, z) \, P^{s_z}(z \mid x_z). \label{eq:target-id}
\end{align}
We begin by expanding $\ex\left[Y^{s_y}_{x_y, W^{s_w}_{x_w}} \mid X = x_z, S=s_z\right]$ as
\begin{align}
    &\sum_z \ex\left[Y^{s_y}_{x_y, W^{s_w}_{x_w}} \mid x_z, z, s_z\right] P(z \mid x_z, s_z) \label{eq:step-1} \\
    &= \sum_z \ex\left[Y^{s_y}_{x_y, w} \mathbbm{1}(W^{s_w}_{x_w}=w) \mid x_z, z, s_z\right] P(z \mid x_z, s_z) \label{eq:step-2} \\
    &= \sum_z \ex\left[Y^{s_y}_{x_y, w} \mid x_z, z, s_z\right] P(W^{s_w}_{x_w}=w \mid x_z, z, s_z) P(z \mid x_z, s_z) \label{eq:step-3},
\end{align}
where Eq.~\ref{eq:step-1} follows from conditioning on $Z$ and marginalizing, Eq.~\ref{eq:step-2} from counterfactual unnesting \citep{correa2025counterfactual}, and Eq.~\ref{eq:step-3} from the independence $Y^{s_y}_{x_y, w} \ci W_{x_w}^{s_w} \mid Z, S, X$ in the counterfactual graph in Fig.~\ref{fig:po-ctf-graph}. 
Note that further that from Fig.~\ref{fig:po-ctf-graph} $Y^{s_y}_{x_y, w} \ci X, S, W \mid Z$, so that
\begin{align}
    \ex\left[Y^{s_y}_{x_y, w} \mid x_z, z, s_z\right] &= \ex\left[Y^{s_y}_{x_y, w} \mid x_y, z, s_y, w\right] \\
    &= \ex\left[Y \mid x_y, z, s_y, w\right] = \ex^{s_y}\left[Y \mid x_y, z, w\right], \label{eq:y-term-id}
\end{align}
where the first step follows from the independence, second from consistency, and the final from the definition.
Further, Fig.~\ref{fig:po-ctf-graph} also implies the independence $W^{s_w}_{x_w} \ci X, S \mid Z$, so that we can write
\begin{align}
    P(W^{s_w}_{x_w}=w \mid x_z, z, s_z) &= P(W^{s_w}_{x_w}=w \mid x_w, z, s_w) \\
    &= P(W=w \mid x_w, z, s_w) = P^{s_w}(W=w \mid x_w, z).
    \label{eq:w-term-id}
\end{align}
where again the first step follows from the independence, second from consistency, and the final from the definition. Plugging in Eq.~\ref{eq:y-term-id} and Eq.~\ref{eq:w-term-id} into Eq.~\ref{eq:step-3} gives Eq.~\ref{eq:target-id}, completing the proof.
\end{proof}

\paragraph{Estimation.}
The identification result in Prop.~\ref{prop:id-est} expresses the target functional in terms of observational distribution conditionals, and therefore the plug-in estimator obtained by fitting each conditional separately is, in principle, consistent. However, plug-in estimators based on flexible nuisance estimators (e.g., gradient-boosted trees or neural networks) generally fail to achieve $\sqrt{n}$-convergence, since the bias of the nuisance estimators enters the final estimator at first order. A standard approach is to derive the influence function of the target functional and perform a one-step debiasing \citep{chernozhukov2018double}.
By performing one-step debiasing for the functional on the RHS of Eq.~\ref{eq:po-id}, we obtain the  estimator:
\begin{align} \label{eq:po-est}
    \widehat{\theta} &= \frac{1}{n} \sum_{i=1}^n f(X_i, Z_i, W_i, Y_i),
\end{align}
where
\begin{align}
    f(X, Z, W, Y) \overset{\Delta}{=}& \frac{\mathbbm{1}(X = x_y)}{\tilde{P}(x_z)} \frac{\tilde{P}(x_z \mid Z)}{\tilde{P}(x_w \mid Z)} \frac{\tilde{P}(x_w \mid Z, W)}{\tilde{P}(x_y \mid Z, W)} \left[Y - \mu(x_y, Z, W)\right] \notag \\
    &+ \frac{\mathbbm{1}(X = x_w)}{\tilde{P}(x_z)} \frac{\tilde{P}(x_z \mid Z)}{\tilde{P}(x_w \mid Z)} \left[\mu(x_y, Z, W) - \ex_{\tilde P}[\mu(x_y, Z, W) \mid x_w, Z]\right] \notag \\
    &+ \frac{\mathbbm{1}(X = x_z)}{\tilde{P}(x_z)} \ex_{\tilde{P}}[\mu(x_y, Z, W) \mid x_w, Z], \notag
\end{align}
where $\mu(x_y, Z, W) = \ex^{s_y}[Y \mid x_y, Z, W]$ is the conditional mean, $\tilde{P}$  denotes the joint distribution $\tilde P(x, z, w, y) = P^{s_y}(y \mid x, z, w) P^{s_w} (w \mid x, z) P^{s_z}(x, z)$ from which the propensities $\tilde P(x_y \mid z, w)$, $\tilde P(x_w \mid z)$, $\tilde P(x_z)$, and $\tilde P(x_z \mid z)$ are obtained. $\ex_{\tilde{P}}[\mu(x_y, Z, W) \mid x_w, Z]$ denotes the $Z$-specific nested mean of $\mu$ integrated out $W$ within the subpopulation $\{X = x_w, S = s_w\}$. The estimator in Eq.~\ref{eq:po-est} is \emph{doubly robust}: it remains consistent whenever either the outcome regressions $\mu, \ex[\mu \mid \cdot]$ or the propensities $\tilde{P}$ are correctly specified, but not necessarily both. In practice, we fit all nuisance functions via gradient boosting with cross-fitting, and plug the resulting estimates into Eq.~\ref{eq:po-est}. Confidence intervals for the estimator are obtained by assuming asymptotic normality holds, with the variance given by the empirical variance of the influence function. 
\newpage
\section{Elicitation of LLM Distributions} \label{appendix:elicit}

In this appendix, we describe in detail the procedure used to sample from the generative model's distributions, which is required for constructing the datasets $\mathcal{D}^{s_0, s_0, s_1}$ (real-world mechanism $f_Y$ replaced by the model's mechanism), $\mathcal{D}^{s_0, s_1, s_1}$ (mechanisms $f_W, f_Y$), and $\mathcal{D}^{s_1}$ (all mechanisms replaced by the model one). 
At a high level, the elicitation consists of two steps: (i) the generator $\Gamma$, which is prompted to produce a natural-language narrative describing an individual, and (ii) the annotator $\mathcal{A}$, which extracts structured covariate values from the generated narrative. The annotation step is performed using Llama 3.3 70B \citep{grattafiori2024llama}, the largest model in our suite (see App.\ref{appendix:annotator-validation} for details on ascertaining annotator validity). We describe the pipeline using the NSDUH dataset as an example, and show how $\mathcal{D}^{s_0, s_1, s_1}$ is generated. BRFSS and ACS Census datasets follow an analogous procedure.

\paragraph{NSDUH Dataset.} The National Survey on Drug Use and Health (NSDUH) is a nationally representative survey conducted annually by the Substance Abuse and Mental Health Services Administration (SAMHSA). We use the data collected in 2023, restricted to adults aged 18 and older living in the United States. The covariates considered are are: age and sex (confounders $Z$), race (protected attribute $X$), education and income (mediators $W$), and monthly marijuana use (outcome $Y$). Survey data provides sample weights $w_i$ that can be used to make the data representative at the national level. For the dataset, we sample $n=8,192$ samples with weights proportional to the survey weights $w_i$, yielding the dataset $\mathcal{D}^{s_0}$ in which all covariates are determined from the real world.
 
\paragraph{Generation.} We illustrate the generation procedure for the stage $\mathcal{D}^{s_0, s_1, s_1}$, in which the values of $X, Z$ are taken from the real-world dataset $\mathcal{D}^{s_0}$ while the values of $W, Y$ are generated by the model. For each sampled row $(X_i, Z_i)$, the generator $\Gamma$ is queried with the following prompt, where the known facts (age, sex, race) are filled in from the real-world data, and the unknown facts (education, income, monthly marijuana use) are to be determined by the model. The prompt makes explicit the population of interest (adults aged 18 and older, living in the United States, in 2023), which provides the model with context on the time frame and geographic location.

\begin{tcolorbox}[colback=gray!10,colframe=gray!50,title=\textbf{Generator Prompt ($\Gamma$)}]
\begin{verbatim}
You are a data generator. Follow the rules strictly.
CONTEXT: We are considering adults (18+ years old)
living in the United States in 2023.
RULES:
1) Write a single narrative enclosed in <story>...</story>.
2) Do NOT include headings, lists, analysis, or
   any text outside the tags.
3) Mention ALL facts given below exactly once (age, sex,
   race, education, income, marijuana use last month).
4) Keep it under 200 words.
known facts to be mentioned:
- age = 30-34 years
- sex = female
- race = White
unknown facts to be mentioned:
- edu (possible values: <= 8th grade, Some high school,
  High school graduate, Some college no degree,
  Associate degree, Bachelor's or higher)
- income (possible values: < $10,000, $10,000 - $19,999,
  $20,000 - $29,999, $30,000 - $39,999, $40,000 - $49,999,
  $50,000 - $74,999, > $75,000)
- mj_monthly (possible values: no, yes)
OUTPUT FORMAT:
<story>
(your narrative here)
</story>
\end{verbatim}
\end{tcolorbox}

\noindent Generation is performed with \texttt{temperature = 1} and \texttt{top\_p = 1}, so that samples reflect the model's distribution in deployment. A typical response from $\Gamma$ looks as follows:

\begin{tcolorbox}[colback=blue!5,colframe=blue!30,title=\textbf{Example Output}]
\textit{As a 30–34-year-old White female living in the United States in 2023, Sarah has achieved a level of success with a Bachelor’s or higher degree in her field. She works diligently, earning an income of \$50,000 - \$74,999 per year. In her free time, Sarah occasionally partakes in recreational activities, including the use of marijuana, and has indeed used it within the last month.}
\end{tcolorbox}

\paragraph{Annotation.} The narrative produced by $\Gamma$ is then passed to the annotator $\mathcal{A}$, which is queried separately for each unknown covariate. For the education variable, the prompt to $\mathcal{A}$ takes the form:

\begin{tcolorbox}[colback=gray!10,colframe=gray!50,title=\textbf{Annotator Prompt ($\mathcal{A}$)}]
\begin{verbatim}
Based on the following text:

"As a 30–34-year-old White female living in the United 
States in 2023, Sarah has achieved a level of success 
with a Bachelor’s or higher degree in her field. She 
works diligently, earning an income of $50,000 - $74,999
per year. In her free time, Sarah occasionally partakes 
in recreational activities, including the use of marijuana,
and has indeed used it within the last month."

determine the person's education. Begin your answer with the
capital letter corresponding to your chosen option below,
followed by a period.
A. <= 8th grade
B. Some high school
C. High school graduate
D. Some college, no degree
E. Associate degree
F. Bachelor's or higher
Output:
\end{verbatim}
\end{tcolorbox}

\noindent Rather than sampling from $\mathcal{A}$, we examine the next-token probabilities assigned to the option letters $\{\texttt{A}, \texttt{B}, \ldots, \texttt{F}\}$, and assign the covariate the value corresponding to the letter with the highest probability. This deterministic argmax over option tokens removes sampling noise in the annotation step, isolating the generator as the sole source of stochasticity in the pipeline. Analogous prompts are used for income and monthly marijuana use, with the option list adapted to the levels of the respective variable. The resulting tuple $(X_i, Z_i, W^{gm}_i, Y_i^{gm})$ is then added to $\mathcal{D}^{s_0, s_1, s_1}$.

\paragraph{Other stages.} The prompts for $\mathcal{D}^{s_0, s_0, s_1}$ and $\mathcal{D}^{s_1}$ follow the same template, with only the partition between \emph{known facts} and \emph{unknown facts} changing:
\begin{itemize}
    \item For $\mathcal{D}^{s_0, s_0, s_1}$ (only $f_Y$ replaced), the known facts include $X, Z$, and $W$ (all taken from the real-world data), while only $Y$ is listed under unknown facts.
    \item For $\mathcal{D}^{s_1}$ (all mechanisms replaced), no known facts are supplied beyond the population context, and all of $X, Z, W, Y$ appear as unknown facts to be generated.
\end{itemize}
The annotation step is identical across stages, with the annotator separately queried for each variable listed as unknown.

\paragraph{BRFSS and ACS Census.}

The BRFSS (Behavioral Risk Factor Surveillance System, 2023) \citep{brfss2023} and American Communicity Survey Census (2023) \citep{acs2023} datasets follow the same procedure. 
The key modifications are: 
\begin{enumerate}[label=(\roman*)]
    \item the population context in the generator prompt is updated:
        \subitem \hspace{-2em}-- BRFSS:  \emph{We are considering adults (18+ years old) living in the United States in 2023.}
        \subitem \hspace{-2em}-- ACS Census:  \emph{We are considering adults (18+ years old) who are employed in the United States in 2023.}
    \item the variable list of known and unknown facts changess, reflecting the dataset-specific SFMs shown in Fig.~\ref{fig:sfms}:
        \subitem \hspace{-2em}-- BRFSS: age and sex (confounders $Z$), race (attribute $X$), education, income, BMI, and monthly exercise (mediators $W$), and diabetes (outcome $Y$).
        \subitem \hspace{-2em}-- ACS Census: age, race, and economic region (confounders $Z$), sex (attribute $X$), education, hours worked, and employer type (mediators $W$), and salary (outcome $Y$).
\end{enumerate}

\subsection{Validating the Annotator LLM} \label{appendix:annotator-validation}
As mentioned previously, the annotation step is performed using $\mathcal{A}=$ Llama 3.3 70B \citep{grattafiori2024llama}, the largest model in our suite, which we found performed best for this task. To validate the annotator, we manually inspected a random subset of generated narratives and compared the human-assigned covariate values against those produced by $\mathcal{A}$.

We aimed for a $\pm 5\%$ confidence interval on the annotator accuracy (pooled across all variables). Using the normal approximation $1.96 \sqrt{p(1-p)/n} \leq 0.05$ at $p \geq 0.9$, this requires $n \geq 139$ ($\approx 14$ samples per model). We drew $n = 140$ narrative-variable pairs at random from the fully-generated stage $\mathcal{D}^{s_1}$ on the NSDUH dataset, stratified across the 10 models we investigated.

For each sampled narrative-variable pair, a human annotator independently assigns the unknown covariate value for $V_i \in \{X, Z, W, Y \}$ based on the text, or marks the text as \emph{inconclusive} if the narrative does not contain enough information to determine the value (e.g., the generator omits the requested fact, or describes it ambiguously). 
We found that $4.9\%$ of narrative-variable pairs were inconclusive, indicating that the generators $\Gamma$ relatively reliably mentioned the requested facts as instructed by the prompt.
On the remaining conclusive pairs, the annotator $\mathcal{A}$ agreed with the human label in $96.4\%$ of cases, which is within the targeted margin of error. We deem this level of agreement sufficient to treat $\mathcal{A}$ as a faithful extractor of structured covariate values from the generator's narratives.
\newpage
\section{Extended Experiments} \label{appendix:empirical}

In this appendix, we provide extended empirical results that expand on the findings reported in the main text. We begin by summarizing the existing disparities present in the real-world, for each of the three datasets. We then provide further detail of the geometry of bias signatures across models, via pairwise $L_1$ distances and a 2D multi-dimensional scaling plot. 
Finally, we provide more detailed analyses for the two case studies discussed in Sec.~\ref{sec:experiments}.

\paragraph{Real-World Baselines.}
\begin{table}[t]
\centering
\setstretch{1.4}
\begin{tabular}{l | c|c|c}
\toprule
\textbf{Dataset} & Direct & Indirect & Spurious \\
\midrule
NSDUH (Marijuana) & \cellcolor{green!40}$-3.8 \%\pm 1.8$ \% & \cellcolor{blue!15}$0.2 \%\pm 0.7$ \% & \cellcolor{blue!40}$1.8 \%\pm 1.0$ \% \\
BRFSS (Diabetes) & \cellcolor{blue!40}$5.7 \%\pm 2.0$ \% & \cellcolor{blue!40}$2.7 \%\pm 1.3$ \% & \cellcolor{green!40}$-5.2 \%\pm 1.3$ \% \\
ACS Census (Income) & \cellcolor{blue!40}$10.4 \%\pm 1.9$ \% & \cellcolor{blue!15}$0.3 \%\pm 1.3$ \% & \cellcolor{blue!15}$0.5 \%\pm 2.7$ \% \\
\bottomrule
\end{tabular}
\vspace{0.1in}
\caption{Real-world discrimination across datasets with 95\% confidence intervals.}
\label{tab:world}
\end{table}

\begin{figure}[t]
\centering
\begin{subfigure}{0.50\linewidth}
\centering
\includegraphics[width=\linewidth]{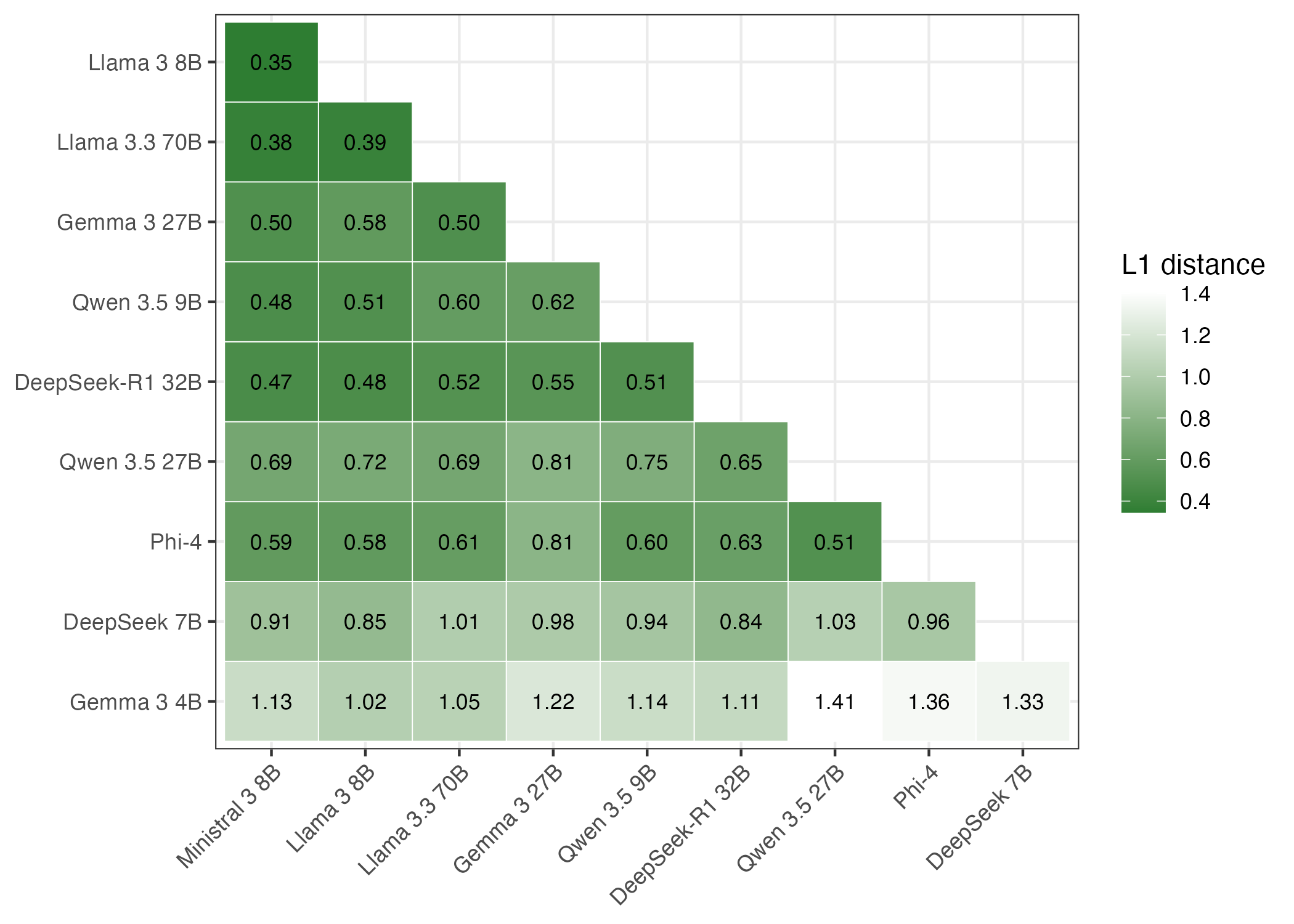}
\caption{Pairwise $L_1$ distances between bias signatures.}
\label{fig:pairwise}
\end{subfigure}
\hfill
\begin{subfigure}{0.48\linewidth}
\centering
\includegraphics[width=\linewidth]{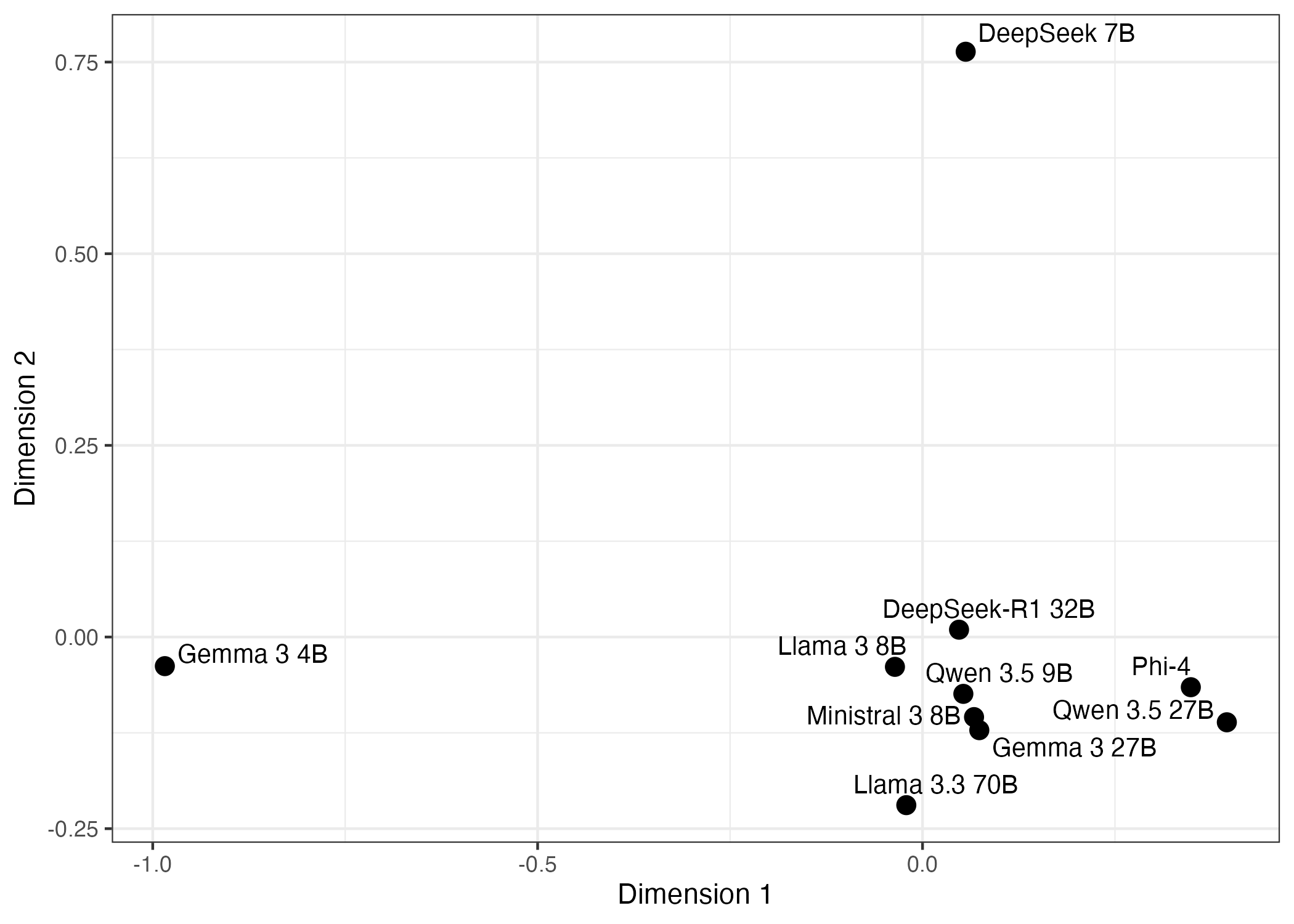}
\vspace{-4pt}
\caption{2D multidimensional scaling of bias signatures.}
\label{fig:mds}
\end{subfigure}
\caption{Similarity of model bias signatures: (a) full pairwise $L_1$ distance matrix, and (b) 2D multidimensional scaling (MDS) embedding.}
\label{fig:proximity}
\end{figure}
Tab.~\ref{tab:world} reports the direct, indirect, and spurious effects in the real-world data $\mathcal{D}^{s_0}$ for each of the three datasets, together with their 95\% confidence intervals. Specifically, the table reports the values of $x\text{-DE}_{x_0, x_1}^{s_0}(y \mid x_0)$, $-x\text{-IE}_{x_1, x_0}^{s_0}(y \mid x_0)$, and $-x\text{-SE}_{x_1, x_0}^{s_0}(y)$, with entries color-coded by direction (blue for disadvantage, green for advantage) and statistical significance (deeper shade for $|x\text{-CE}^{s_0}| \geq 1.96 \cdot \sigma$). 
Negative values of indirect and spurious effects are considered, since these effects consider reverse transitions $x_1 \to x_0$, and reporting the negative values maintains equivalent directionality of interpretation as for the direct effect.
The reported values provide the baseline against which the model-replacement influences are measured. We provide a summary of the bias statistics in the real world.
On NSDUH (outcome $Y$ is monthly marijuana use), the direct effect is significantly negative ($-3.8\% \pm 1.8\%$): controlling for sex, age, education, and income, minority individuals are \emph{less} likely to use marijuana monthly than majority individuals. 
The indirect effect mediated by education and income is close to zero and not significant ($0.2\% \pm 0.7\%$), meaning that the effect mediated by education and income is not strong. 
The spurious effect is positive and significant ($1.8\% \pm 1.0\%$), driven by the younger age profile of the minority cohort, which is itself associated with higher marijuana use. 
On BRFSS ($Y$ is diabetes diagnosis), the direct effect is significantly positive ($5.7\% \pm 2.0\%$), implying that minority individuals carry an elevated risk of diabetes that is not explained by socioeconomic and behavioral mediators. 
The indirect effect is also positive and significant ($2.7\% \pm 1.3\%$), reflecting that differences in education, income, BMI, and exercise further increase disparity, while the spurious effect is significantly negative ($-5.2\% \pm 1.3\%$), again driven by the cohort's younger age profile -- here a protective factor against diabetes. 
On ACS Census (outcome $Y$ is earning less than \$$50$k/year), the direct effect dominates ($10.4\% \pm 1.9\%$): controlling for age, education, region, hours worked, and employer type, women are more likely to be low-earners. 
Indirect ($0.3\% \pm 1.3\%$) and spurious ($0.5\% \pm 2.7\%$) effects are both small and not significant, indicating that the disparity does not flow through observed mediators or confounders. 

\paragraph{Similarity of Bias Signatures.}
Fig.~\ref{fig:proximity} complements the dendrogram in Fig.~\ref{fig:dendrogram} of the main text with two additional views of the same similarity structure: the full pairwise $L_1$ distance matrix (Fig.~\ref{fig:pairwise}) and a 2D multidimensional scaling (MDS) embedding (Fig.~\ref{fig:mds}). The heatmap shows that most models occupy a relatively tight central cluster, with pairwise distances among Llama 3 8B, Llama 3.3 70B, Ministral 3 8B, Qwen 3.5 9B, and DeepSeek-R1 32B falling at $0.6$ or below. Gemma 3 4B and DeepSeek 7B sit further from the rest of the suite, with most of their pairwise distances exceeding $0.9$. The MDS projection reinforces this picture: the bulk of the models cluster together near the origin, while Gemma 3 4B separates along the first dimension and DeepSeek 7B along the second.

\subsection{Extended Case Studies}
We now analyze the case studies from the main text using waterfall plots. Specifically, we look at Gemma 3 27B on NSDUH and Qwen 3.5 27B on BRFSS. 
Each waterfall plot visualizes the decomposition in Eq.~\ref{eq:delta-ce-decomp}: the leftmost bar is the real-world effect $\text{CE}^{s_0}$, the three middle bars correspond to the successive replacements of mechanisms $f_Y$, $f_W$, and $f_{X,Z}$ by the model mechanisms (giving the three $\Delta x\text{-CE}$ quantities in Thm.~\ref{thm:new-decomp}), and the rightmost bar is the fully-replaced effect $\text{CE}^{s_1}$.

\subsubsection{Gemma 3 27B on NSDUH}

\begin{figure}[t]
\centering
\begin{subfigure}{0.48\linewidth}
\centering
\includegraphics[width=\linewidth]{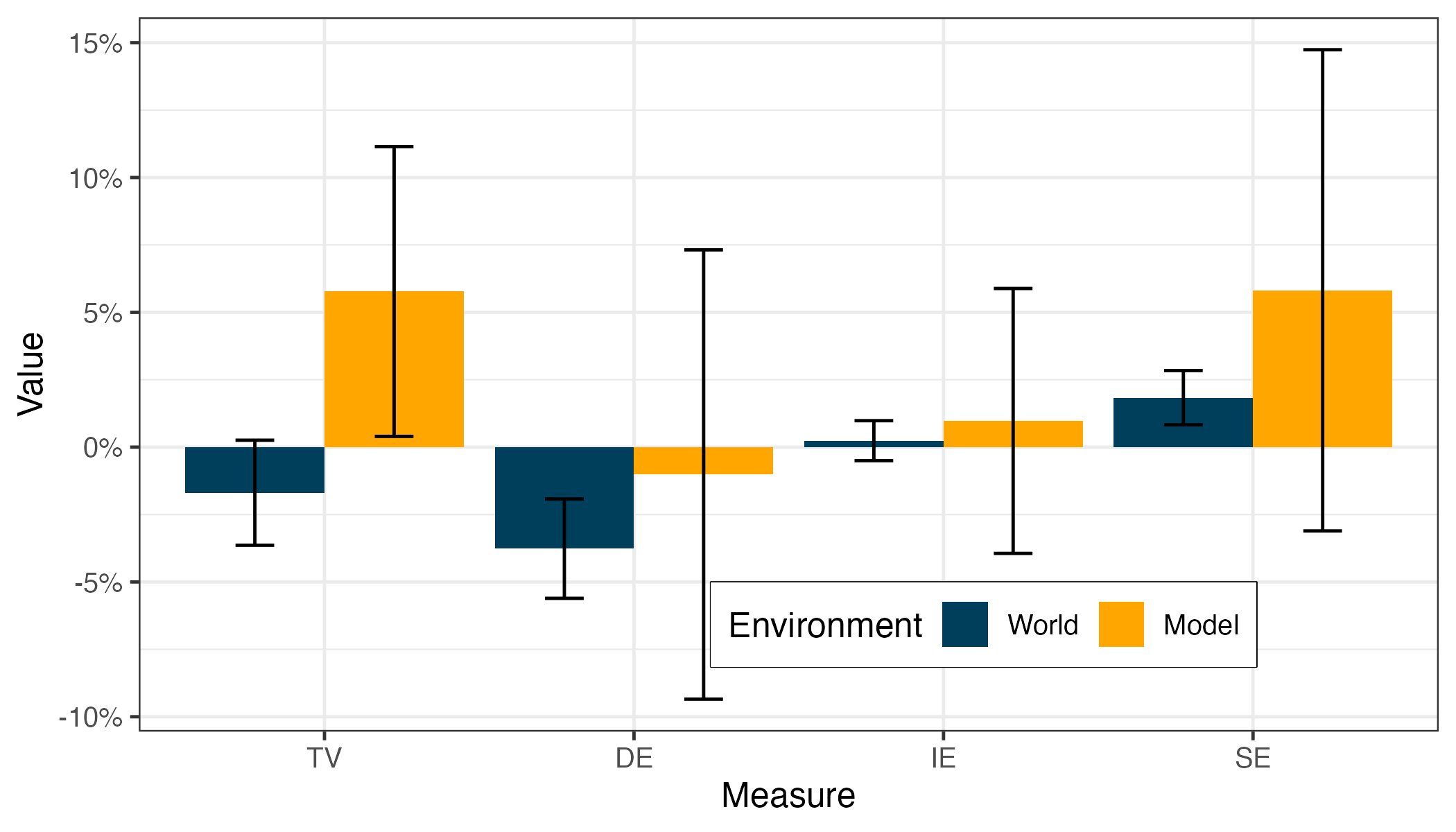}
\caption{TV decompositions.}
\label{fig:tv-decomp-gemma}
\end{subfigure}
\hfill
\begin{subfigure}{0.48\linewidth}
\centering
\includegraphics[width=\linewidth]{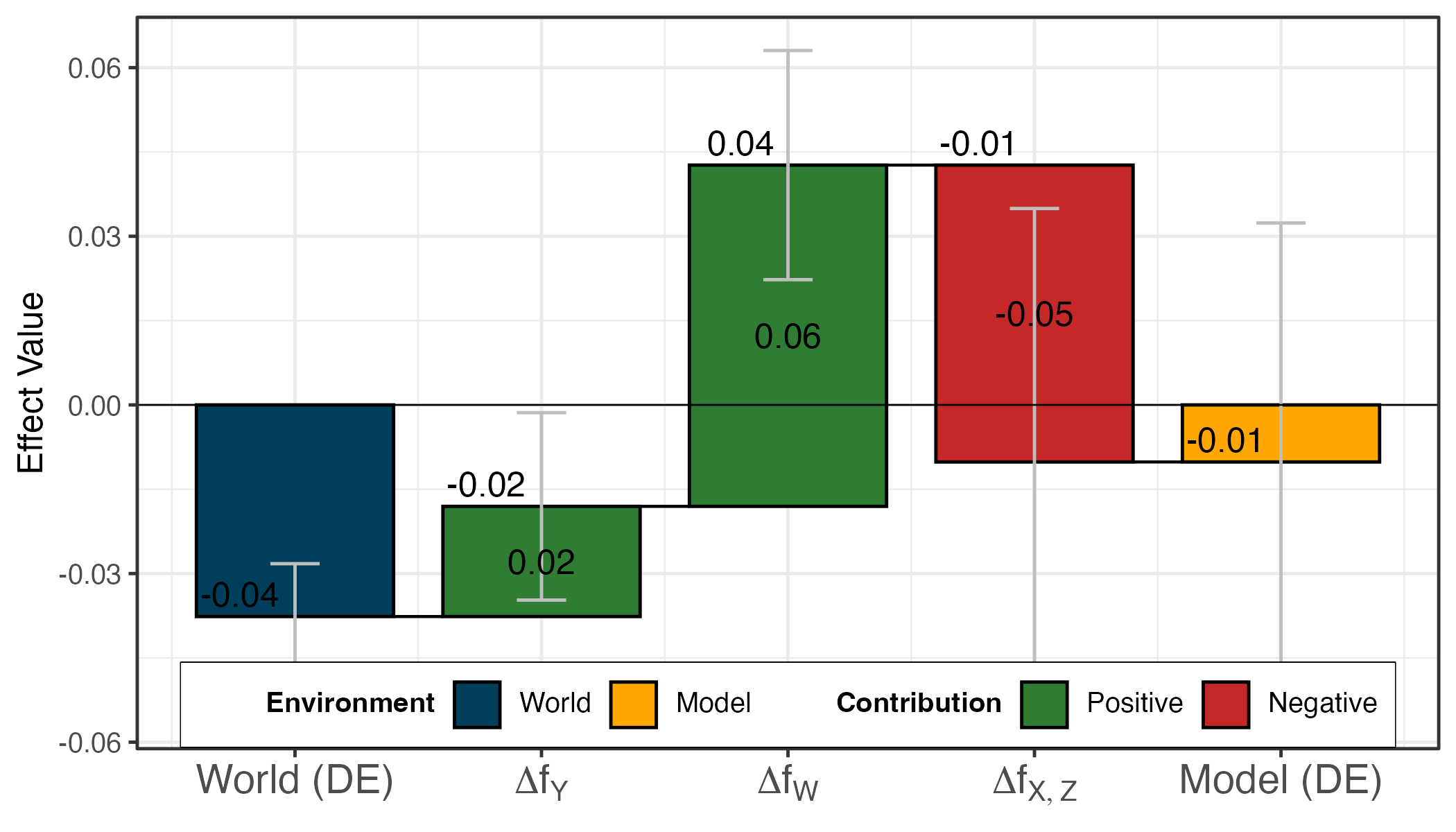}
\caption{DE waterfall.}
\label{fig:wfall-gemma-de}
\end{subfigure}
\hfill
\begin{subfigure}{0.48\linewidth}
\centering
\includegraphics[width=\linewidth]{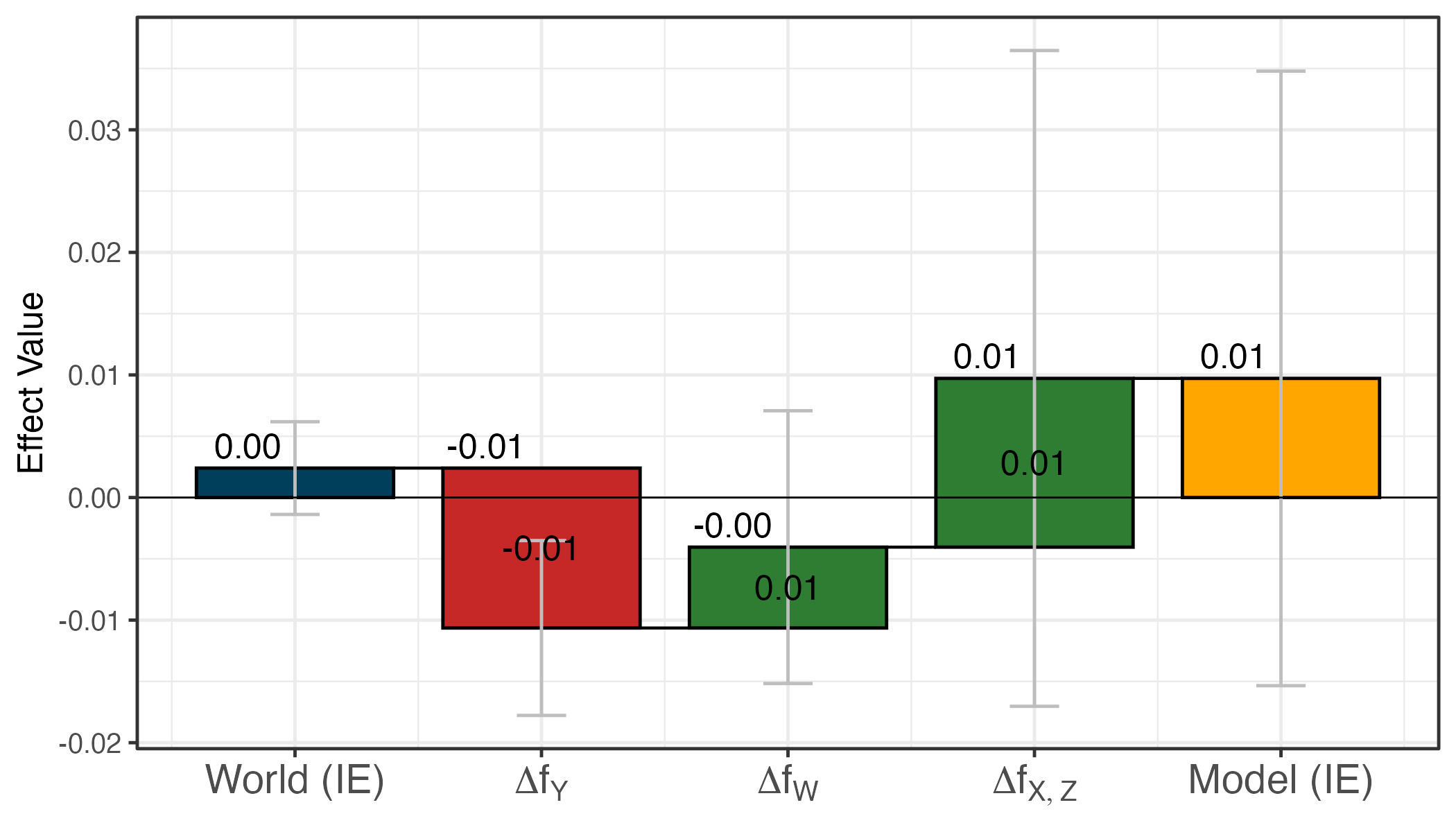}
\caption{IE waterfall.}
\label{fig:wfall-gemma-ie}
\end{subfigure}
\hfill
\begin{subfigure}{0.48\linewidth}
\centering
\includegraphics[width=\linewidth]{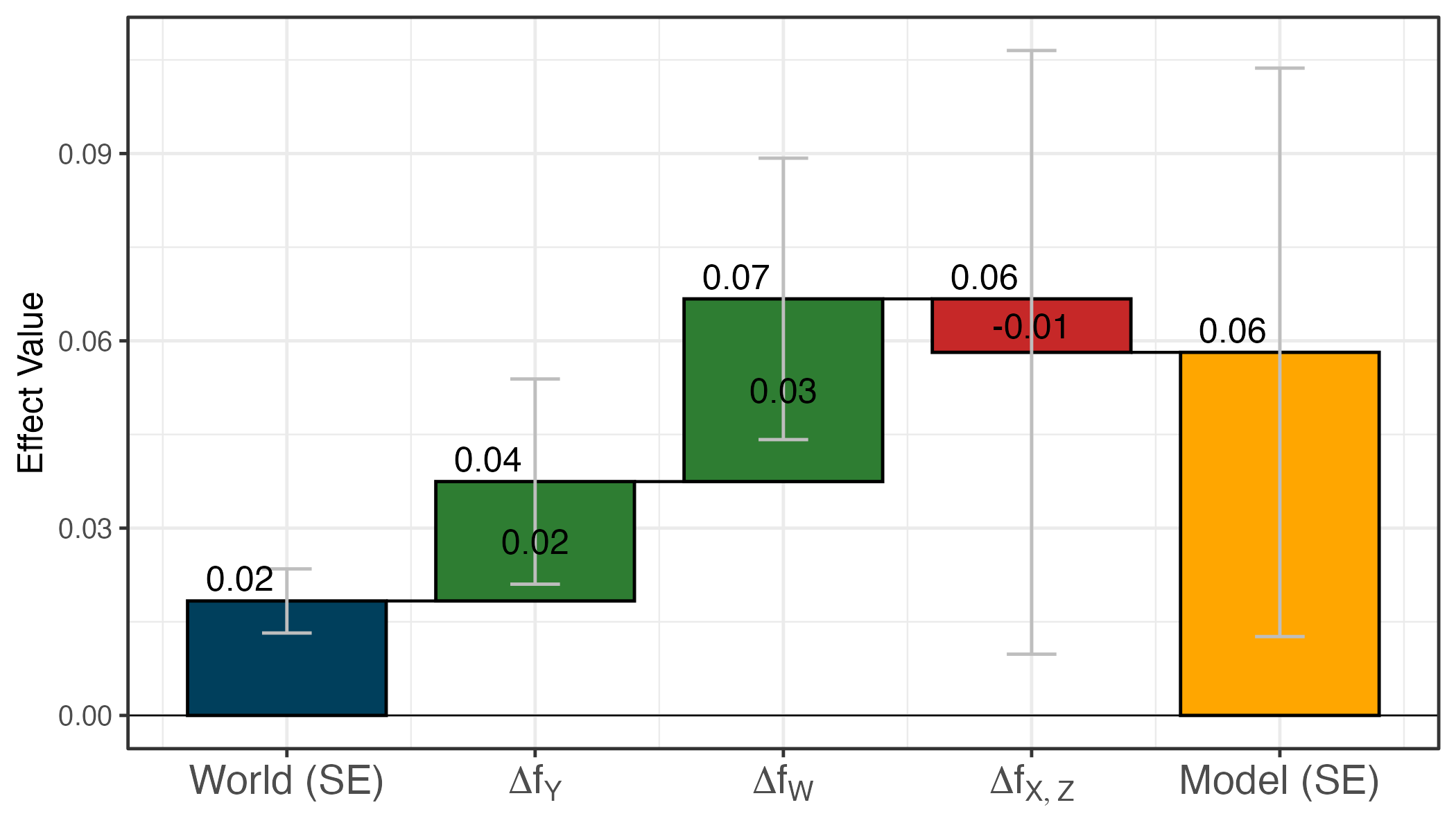}
\caption{SE waterfall.}
\label{fig:wfall-gemma-se}
\end{subfigure}
\caption{TV decomposition into $\Delta x\text{-DE}$, $\Delta x\text{-IE}$, and $\Delta x\text{-SE}$ for Gemma 3 27B on NSDUH.}
\label{fig:wfall-gemma}
\end{figure}

We now examine each pathway for Gemma 3 27B on NSDUH (Fig.~\ref{fig:wfall-gemma}), illustrating how the decomposition of Thm.~\ref{thm:new-decomp} traces, step by step, how each mechanism replacement reshapes the corresponding causal pathway between the real world and the model.

\paragraph{Direct effect.} 
The first column of Fig.~\ref{fig:wfall-gemma-de} shows the real-world direct effect, $\text{DE}^{s_0} = -3.8\% \pm 1.8\%$: controlling for sex, age, education, and income, minority individuals are \emph{less} likely to use marijuana monthly than majority individuals. The replacement $f_Y^{s_0} \to f_Y^{s_1}$ (second column, $\Delta\text{-DE}^{s_0, s_0, s_0 \to s_1} = +2.0\% \pm 3.3\%$) shifts the direct effect upward but does not yet reverse its sign ($-1.8\% \pm 2.7\%$). The subsequent replacement $f_W^{s_0} \to f_W^{s_1}$ contributes a substantial further upward shift ($+6.1\% \pm 4.0\%$), tipping the direct effect into significant positive territory ($+4.3\% \pm 2.9\%$): under Gemma's $f_Y$, the direct effect is sensitive to the distribution of $W$, and Gemma's $f_W$ produces a $W \mid X$ distribution that pushes the direct effect toward stereotyping minorities as more likely to use marijuana. Finally, the $f_{X,Z}$ replacement shifts the direct effect by $-5.3\% \pm 8.8\%$, and the fully-replaced $\text{DE}^{s_1} = -1.0\% \pm 8.3\%$ no longer reaches statistical significance.

\paragraph{Indirect effect.} 
The real-world indirect effect is close to zero and not significant, $\text{IE}^{s_0} = 0.2\% \pm 0.7\%$: differences in education and income between minority and majority individuals do not translate into a meaningful shift in marijuana use. Replacing $f_Y$ shifts the indirect effect to $-1.1\% \pm 1.2\%$ ($\Delta = -1.3\% \pm 1.4\%$), the $f_W$ replacement contributes a small positive shift ($+0.7\% \pm 2.2\%$) leaving the indirect effect at $-0.4\% \pm 1.8\%$, and the $f_{X,Z}$ replacement adds a further $+1.4\% \pm 5.2\%$, with the model's indirect effect landing at $\text{IE}^{s_1} = +1.0\% \pm 4.9\%$. None of the intermediate or final values are significant at this sample size; the indirect pathway is the least affected by Gemma's mechanism replacements.

\paragraph{Spurious effect.} 
The real-world spurious effect is $\text{SE}^{s_0} = +1.8\% \pm 1.0\%$, driven by the younger age profile of the minority cohort, which is itself associated with higher marijuana use. Both the $f_Y$ replacement ($+1.9\% \pm 3.2\%$) and the $f_W$ replacement ($+2.9\% \pm 4.4\%$) shift the spurious effect further upward, amplifying the real-world disparity to $\text{SE}^{s_0, s_1, s_1} = +6.7\% \pm 3.2\%$. The subsequent $f_{X,Z}$ replacement contributes only a small downward shift ($-0.9\% \pm 9.5\%$), and the fully-replaced $\text{SE}^{s_1} = +5.8\% \pm 8.9\%$ remains positive though no longer significant due to estimation variance. The dominant contribution here comes from $f_Y$ and $f_W$ jointly, rather than from $f_{X,Z}$ as one might initially expect for a spurious pathway: Gemma's beliefs about how demographics interact with the outcome and mediators amplify the role of demographic confounders relative to the real-world data.

Across all three pathways, the waterfall decomposition makes visible \emph{which} mechanisms are responsible for each step of the shift -- in this case, $f_Y$ and $f_W$ together for both the direct and spurious pathways -- providing a level of granularity that the aggregate $\text{DE}^{s_1}$, $\text{IE}^{s_1}$, $\text{SE}^{s_1}$ values could not.

\subsubsection{Qwen 3.5 27B on BRFSS}

\begin{figure}[t]
\centering
\begin{subfigure}{0.48\linewidth}
\centering
\includegraphics[width=\linewidth]{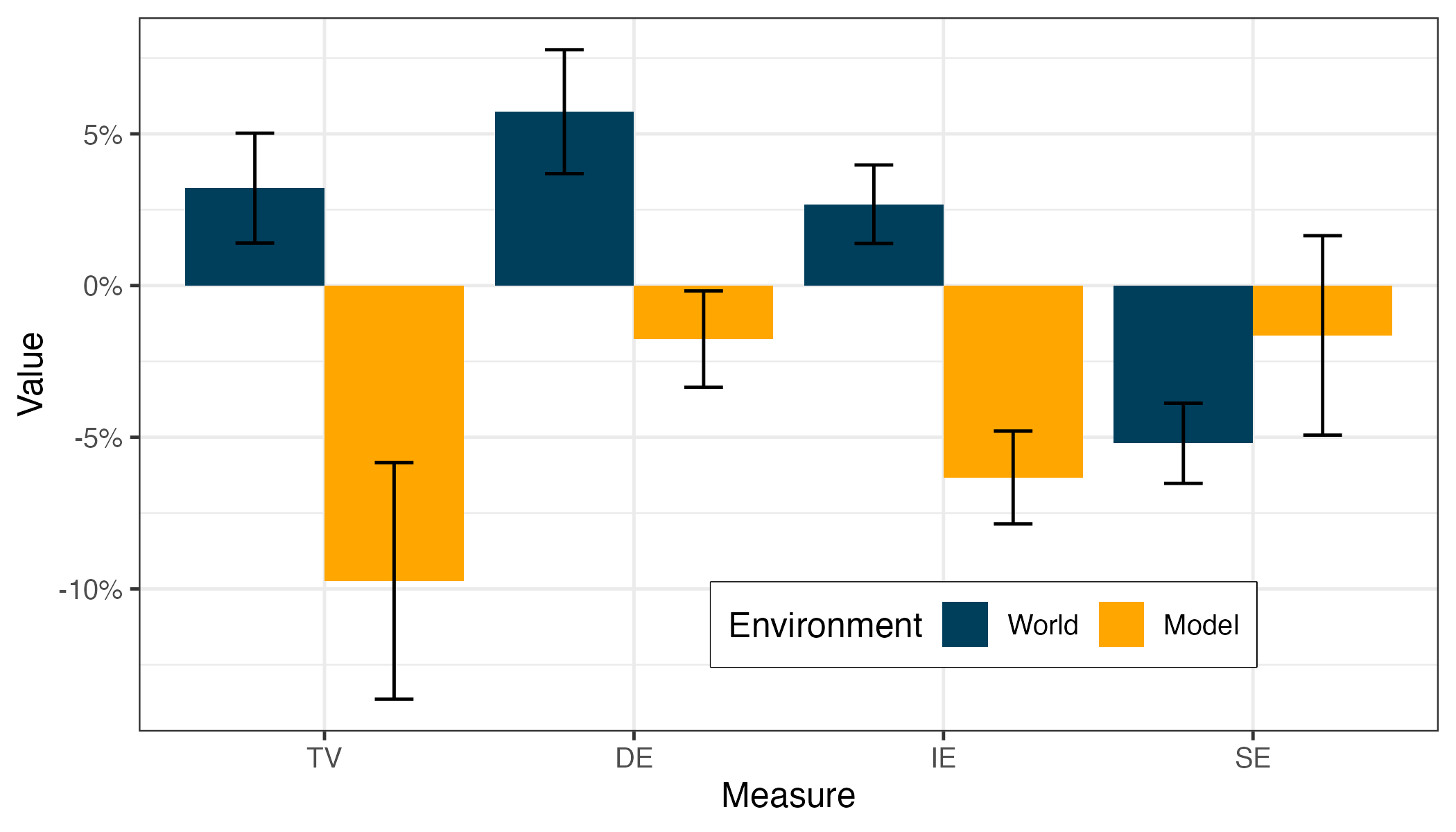}
\caption{TV decompositions.}
\label{fig:tv-decomp-qwen}
\end{subfigure}
\hfill
\begin{subfigure}{0.48\linewidth}
\centering
\includegraphics[width=\linewidth]{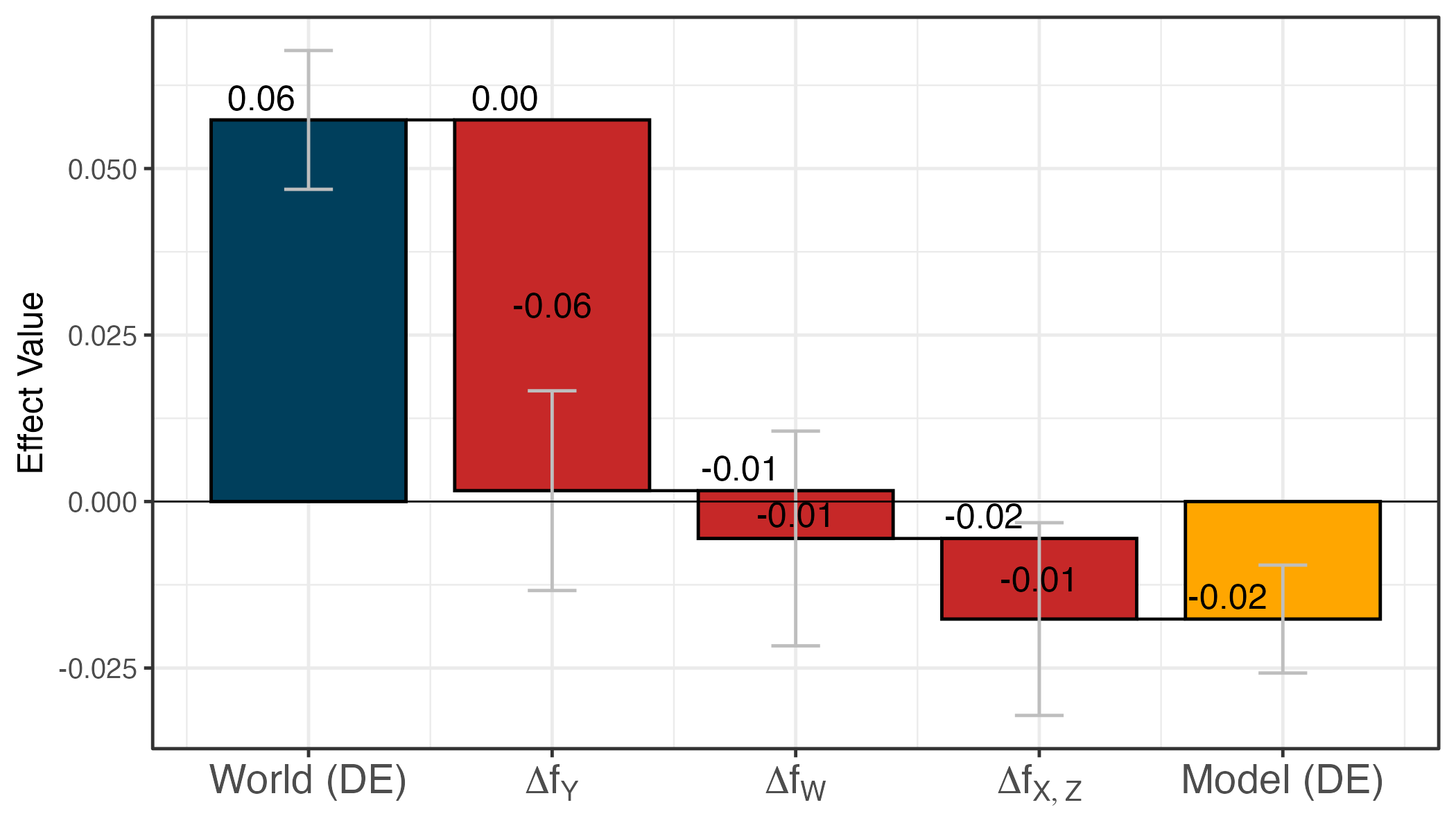}
\caption{DE waterfall.}
\label{fig:wfall-qwen-de}
\end{subfigure}
\hfill
\begin{subfigure}{0.48\linewidth}
\centering
\includegraphics[width=\linewidth]{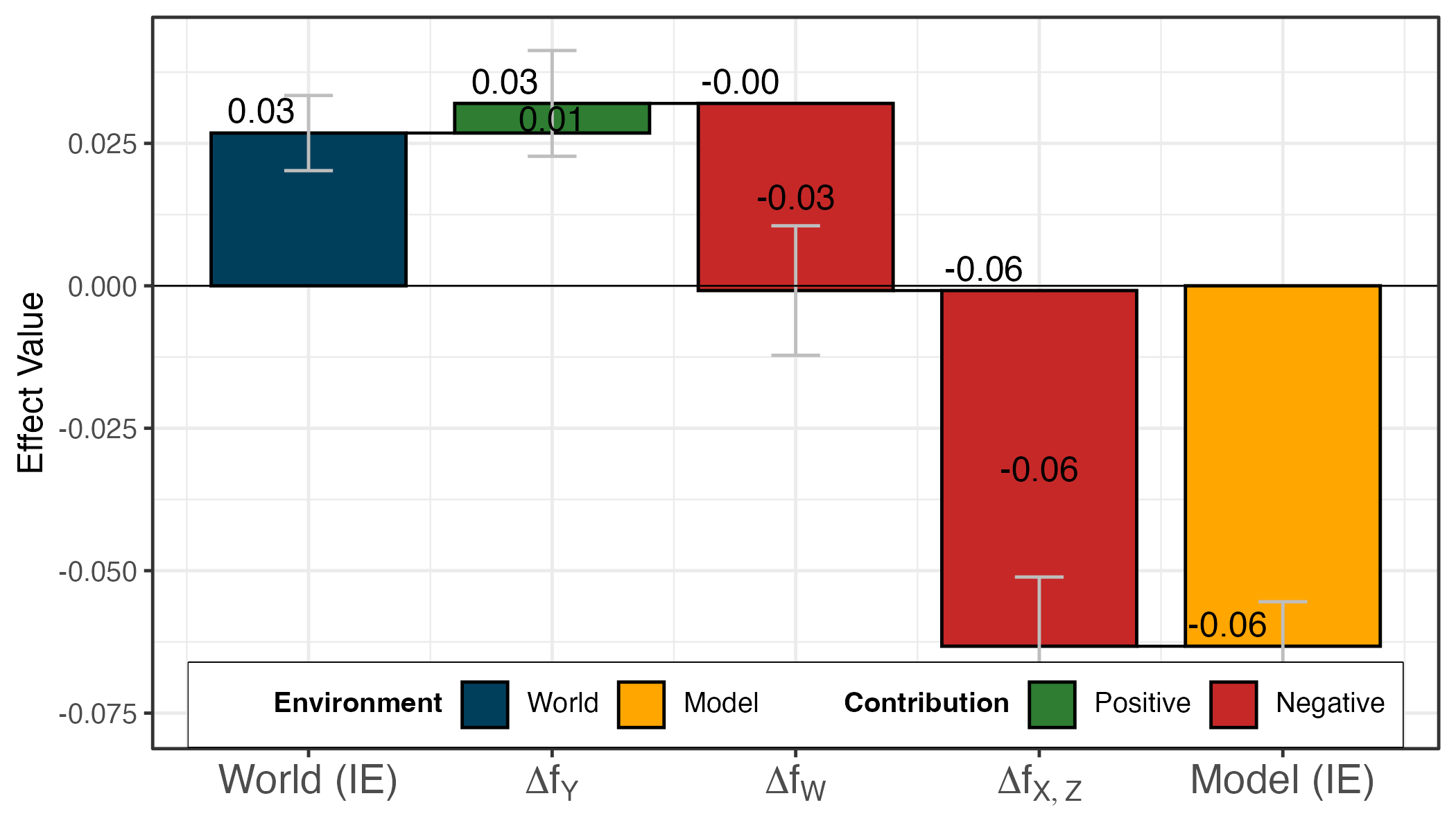}
\caption{IE waterfall.}
\label{fig:wfall-qwen-ie}
\end{subfigure}
\hfill
\begin{subfigure}{0.48\linewidth}
\centering
\includegraphics[width=\linewidth]{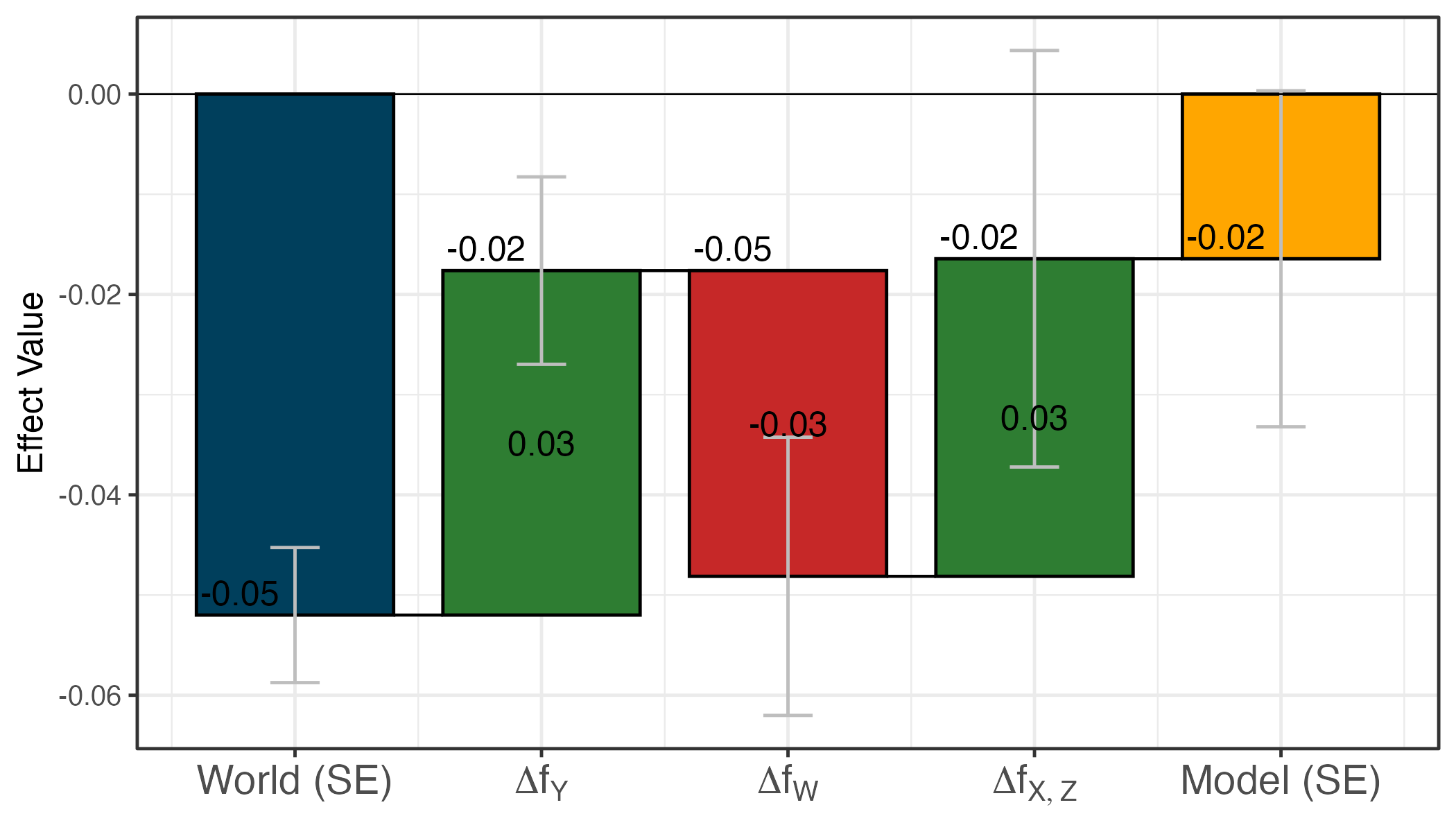}
\caption{SE waterfall.}
\label{fig:wfall-qwen-se}
\end{subfigure}
\caption{TV decomposition into $\Delta x\text{-DE}$, $\Delta x\text{-IE}$, and $\Delta x\text{-SE}$ for Qwen 3.5 27B on BRFSS.}
\label{fig:wfall-qwen}
\end{figure}

We now turn to the corresponding decomposition for Qwen 3.5 27B on BRFSS (Fig.~\ref{fig:wfall-qwen}).

\paragraph{Direct effect.} 
The first column of Fig.~\ref{fig:wfall-qwen-de} shows the real-world direct effect, $\text{DE}^{s_0} = +5.7\% \pm 2.0\%$: minority individuals carry an elevated direct risk of diabetes that is not explained by socioeconomic and behavioral mediators. The replacement $f_Y^{s_0} \to f_Y^{s_1}$ contributes $\Delta\text{-DE}^{s_0, s_0, s_0 \to s_1} = -5.6\% \pm 2.9\%$, bringing the direct effect close to zero ($+0.2\% \pm 2.1\%$): Qwen 3.5 27B's outcome mechanism does not encode a substantial direct dependence of diabetes risk on race. The subsequent $f_W$ and $f_{X,Z}$ replacements each contribute small further negative shifts ($-0.7\% \pm 3.2\%$ and $-1.2\% \pm 2.8\%$), pushing the direct effect to $\text{DE}^{s_1} = -1.8\% \pm 1.6\%$, a significant reversal in sign relative to the real world. The cumulative reversal is dominated by the $f_Y$ replacement -- it is Qwen's outcome mechanism, more than its beliefs about $W$ or $(X, Z)$, that fails to encode the direct race-diabetes association.

\paragraph{Indirect effect.} 
The real-world indirect effect is $\text{IE}^{s_0} = +2.7\% \pm 1.3\%$, reflecting that differences in education, income, BMI, and exercise between minority and majority individuals further amplify the disparity in diabetes risk. The waterfall plot uncovers an important pattern: the $f_Y$ replacement leaves the indirect effect roughly unchanged ($\Delta = +0.5\% \pm 1.8\%$, landing at $+3.2\% \pm 1.3\%$), but the $f_W$ replacement contributes a downward shift ($-3.3\% \pm 2.2\%$), bringing the indirect effect to essentially zero ($-0.1\% \pm 1.8\%$). The $f_{X,Z}$ replacement then contributes a further downward shift, $-6.2\% \pm 2.4\%$, dominating the decomposition and driving the indirect effect to $\text{IE}^{s_1} = -6.3\% \pm 1.5\%$ (Fig.~\ref{fig:wfall-qwen-ie}). The reversal here is jointly driven by $f_W$ and $f_{X,Z}$: Qwen's beliefs about how mediators depend on race, and about the joint distribution of race-correlated covariates, both push the indirect pathway in the opposite direction from reality.

\paragraph{Spurious effect.} 
The real-world spurious effect is $\text{SE}^{s_0} = -5.2\% \pm 1.3\%$, an advantage for the minority cohort driven by its younger age profile (younger age is protective against diabetes). The $f_Y$ replacement dampens this advantage ($\Delta = +3.4\% \pm 1.8\%$, landing at $-1.8\% \pm 1.3\%$), while the $f_W$ replacement partially restores it ($-3.1\% \pm 2.7\%$, bringing $\text{SE}$ back to $-4.8\% \pm 2.4\%$). Finally, the $f_{X,Z}$ replacement again dampens it ($+3.2\% \pm 4.1\%$), with the fully-replaced $\text{SE}^{s_1} = -1.6\% \pm 3.3\%$ no longer distinguishable from zero (Fig.~\ref{fig:wfall-qwen-se}). The trajectory is non-monotone but the cumulative effect is a substantial dampening of the cohort-age advantage: under Qwen's beliefs, the protective effect of younger age for minorities is no longer encoded.

\ifnum\treportflag=0
    \newpage
\section*{NeurIPS Paper Checklist}

\begin{enumerate}

\item {\bf Claims}
    \item[] Question: Do the main claims made in the abstract and introduction accurately reflect the paper's contributions and scope?
    \item[] Answer: \answerYes{}
    \item[] Justification: The four contributions listed in the introduction are each substantiated: the theoretical framework (Sec.~\ref{sec:cfgai}), decomposition results (Thm.~\ref{thm:new-decomp}), identification and estimation (Prop.~\ref{prop:id-est}), and empirical analysis of LLMs for bias (Sec.~\ref{sec:experiments}).

\item {\bf Limitations}
    \item[] Question: Does the paper discuss the limitations of the work performed by the authors?
    \item[] Answer: \answerYes{}
    \item[] Justification: A dedicated ``Limitations and Future Work'' paragraph discusses the key limitations of the framework explicitly.

\item {\bf Theory assumptions and proofs}
    \item[] Question: For each theoretical result, does the paper provide the full set of assumptions and a complete (and correct) proof?
    \item[] Answer: \answerYes{}
    \item[] Justification: All assumptions are stated alongside the theoretical results; proofs for Thm.~\ref{thm:new-decomp}, Cor.~\ref{cor:standard-ml} and Prop.~\ref{prop:id-est} are provided in Appendix~\ref{appendix:proofs}.

    \item {\bf Experimental result reproducibility}
    \item[] Question: Does the paper fully disclose all the information needed to reproduce the main experimental results of the paper to the extent that it affects the main claims and/or conclusions of the paper (regardless of whether the code and data are provided or not)?
    \item[] Answer: \answerYes{}
    \item[] Justification: The elicitation procedure, prompts, discretization details, and model versions are described in Appendix~\ref{appendix:elicit}. Full code is provided in an anonymized repository, with computational details and hyperparameters described.

\item {\bf Open access to data and code}
    \item[] Question: Does the paper provide open access to the data and code, with sufficient instructions to faithfully reproduce the main experimental results, as described in supplemental material?
    \item[] Answer: \answerYes{}
    \item[] Justification: The code is publicly available in an anonymized repository; all datasets used are publicly available and the preprocessing scripts are included.

\item {\bf Experimental setting/details}
    \item[] Question: Does the paper specify all the training and test details (e.g., data splits, hyperparameters, how they were chosen, type of optimizer) necessary to understand the results?
    \item[] Answer: \answerYes{}
    \item[] Justification: All experimental details (model versions, inference setup via vLLM, elicitation prompts, variable discretization) are described in the main text and Appendix~\ref{appendix:elicit}.

\item {\bf Experiment statistical significance}
    \item[] Question: Does the paper report error bars suitably and correctly defined or other appropriate information about the statistical significance of the experiments?
    \item[] Answer: \answerYes{}
    \item[] Justification: Standard deviations are reported alongside all experimental results, and the paper discusses how uncertainty enters the estimation of the causal quantities.

\item {\bf Experiments compute resources}
    \item[] Question: For each experiment, does the paper provide sufficient information on the computer resources (type of compute workers, memory, time of execution) needed to reproduce the experiments?
    \item[] Answer: \answerYes{}
    \item[] Justification: The paper provides this information at the start of the supplementary material (four NVIDIA A100 40GB GPUs, vLLM inference, total compute under 72 hours).
    
\item {\bf Code of ethics}
    \item[] Question: Does the research conducted in the paper conform, in every respect, with the NeurIPS Code of Ethics \url{https://neurips.cc/public/EthicsGuidelines}?
    \item[] Answer: \answerYes{}
    \item[] Justification: The research complies with the NeurIPS Code of Ethics.

\item {\bf Broader impacts}
    \item[] Question: Does the paper discuss both potential positive societal impacts and negative societal impacts of the work performed?
    \item[] Answer: \answerYes{}
    \item[] Justification: A dedicated Broader Impacts Statement in the appendix discusses both positive impacts (enabling detection and mitigation of demographic bias in generative AI) and potential negative impacts (misuse of bias quantification tools).
    
\item {\bf Safeguards}
    \item[] Question: Does the paper describe safeguards that have been put in place for responsible release of data or models that have a high risk for misuse (e.g., pre-trained language models, image generators, or scraped datasets)?
    \item[] Answer: \answerNA{}
    \item[] Justification: The paper introduces a bias detection methodology for existing models. It does not release new models or datasets, posing no such risks.

\item {\bf Licenses for existing assets}
    \item[] Question: Are the creators or original owners of assets (e.g., code, data, models), used in the paper, properly credited and are the license and terms of use explicitly mentioned and properly respected?
    \item[] Answer: \answerYes{}
    \item[] Justification: All assets used are properly cited; all are publicly available under open or permissive licenses.

\item {\bf New assets}
    \item[] Question: Are new assets introduced in the paper well documented and is the documentation provided alongside the assets?
    \item[] Answer: \answerYes{}
    \item[] Justification: The released code implementing the methodology is documented with a README explaining the setup and reproduction steps.

\item {\bf Crowdsourcing and research with human subjects}
    \item[] Question: For crowdsourcing experiments and research with human subjects, does the paper include the full text of instructions given to participants and screenshots, if applicable, as well as details about compensation (if any)?
    \item[] Answer: \answerNA{}
    \item[] Justification: No crowdsourcing or human subjects involved.

\item {\bf Institutional review board (IRB) approvals or equivalent for research with human subjects}
    \item[] Question: Does the paper describe potential risks incurred by study participants, whether such risks were disclosed to the subjects, and whether Institutional Review Board (IRB) approvals (or an equivalent approval/review based on the requirements of your country or institution) were obtained?
    \item[] Answer: \answerNA{}
    \item[] Justification: No human subjects or study participants.

\item {\bf Declaration of LLM usage}
    \item[] Question: Does the paper describe the usage of LLMs if it is an important, original, or non-standard component of the core methods in this research? Note that if the LLM is used only for writing, editing, or formatting purposes and does \emph{not} impact the core methodology, scientific rigor, or originality of the research, declaration is not required.
    \item[] Answer: \answerNA{}
    \item[] Justification: LLMs were used for text polishing, but were not used to generate the paper's core methodology or developments. Accordingly, no declaration is provided.

\end{enumerate}
\fi

\end{document}